\documentclass[compsoc]{IEEEtran}
\usepackage[T1]{fontenc}
\usepackage[
    backend=biber,
    style=ieee,
    sorting=none
]{biblatex}

\usepackage{amsmath,amsfonts, amssymb}
\usepackage{mathrsfs}
\usepackage{nameref}
\usepackage{varioref}
\usepackage{hyperref}
\usepackage{cleveref}
\newtheorem{definition}{\textbf{Definition}}
\newtheorem{theorem}{\textbf{Theorem}}
\newtheorem{proposition}{\textbf{Proposition}}
\usepackage{bm}
\usepackage{multirow} 
\newenvironment{proof}{\textit{Proof.}}{\hfill$\square$}
\usepackage{chngcntr}
\usepackage{apptools}
\AtAppendix{\counterwithin{lemma}{section}}

\DeclareMathOperator*{\argmax}{argmax}

\usepackage{enumitem}
\usepackage{graphicx}
\PassOptionsToPackage{normalem}{ulem}
\usepackage{ulem}
\usepackage[dvipsnames]{xcolor}
\renewcommand{\arraystretch}{1.2}
\definecolor{DarkRed}{RGB}{192, 0, 0}
\definecolor{DarkYellow}{RGB}{220, 159, 18}
\providecolor{added}{rgb}{1,0,0}
\providecolor{deleted}{rgb}{1,0,0}

\usepackage{algorithm}
\usepackage[noend]{algpseudocode}

\usepackage{tikz}
\usepackage{xcolor}

\algrenewcommand\algorithmicforall{\textbf{foreach}}
\algrenewcommand\algorithmicindent{.8em}
\usepackage{pifont}

\usepackage{array}
\usepackage[caption=false,font=normalsize,labelfont=sf,textfont=sf]{subfig}
\usepackage{textcomp}
\usepackage{stfloats}
\usepackage{url}
\usepackage{verbatim}
\usepackage{lipsum}
\usepackage{float}
\usepackage{multicol}
\graphicspath{{./}{../}{../images/}{../biopics/}{images/}{biopics/}}
\renewcommand\thesectiondis{\arabic{section}}
\renewcommand\thesubsectiondis{\thesectiondis.\arabic{subsection}}
\renewcommand\thesubsubsectiondis{\thesubsectiondis.\arabic{subsubsection}}

\begin{document}

\title{Beyond Convolution: Advancing Hypergraph Neural Networks with Hypergraph U-Nets}

\author{Fuli Wang,~\IEEEmembership{Student Member,~IEEE,} Wei Qian, Daniel L Lau,~\IEEEmembership{Fellow, ~IEEE,} \\and Gonzalo R. Arce,~\IEEEmembership{Life Fellow,~IEEE}
\thanks{F. Wang is with the Institute for Financial Services Analytics, University of Delaware, Newark, DE, 19716, USA\\
\indent W. Qian is with the Department of Applied Economics and Statistics, University of Delaware, Newark, DE, 19716, USA\\
\indent D. L. Lau is with the Department of Electrical and Computer Engineering,
University of Kentucky, Lexington, KY 40506 USA\\
\indent G. R. Arce is with the Department of Electrical and Computer Engineering, University of Delaware, Newark, DE, 19716, USA\\

\indent This work was partially supported by the National Science Foundation under grants 1815992, 1816003, and 2413833, the AFOSR award FA9550-22-1-0362 and by the Institute of Financial Services Analytics, co-sponsored by JP Morgan Chase \& Co.}}


\markboth{Journal of \LaTeX\ Class Files,~Vol.~14, No.~8, August~2021}%
{Shell \MakeLowercase{\textit{et al.}}: A Sample Article Using IEEEtran.cls for IEEE Journals}


\maketitle

\begin{abstract}Convolutions have successfully transitioned from image processing to the complex realm of non-Euclidean higher-order domains, particularly in hypergraphs. Despite the success in convolution, the exploration of a popular architecture named U-Net   remains largely unexplored for hypergraph data due to the lack of well-defined pooling and unpooling operations. This work pioneers the study of U-Net architectures for hypergraph data, addressing the critical challenge of designing effective pooling and unpooling operations that retain maximal structural information from the input hypergraph. Motivated by hierarchical clustering, we propose to construct the pooling and unpooling operators all at once by cutting the clustering dendrogram at different granularities, named the Parallel Hierarchical Pooling (PHPool) and Unpooling (PHUnpool) operators. Unlike existing pooling methods that risk local structural damage through a sequential learning procedure, our PHPool operators are designed in a global and parallel manner to ensure fidelity to the original hypergraph structure with efficient computation while the PHUnpool operators are tailored to perform inverse operations of the PHPools for hypergraph reconstruction. We validate our model through hypergraph reconstruction simulation, hypergraph classification, and node-level anomaly detection, where it demonstrates superior performance over existing state-of-the-art graph and hypergraph deep learning methods.

\end{abstract}


\section{Introduction}
\IEEEPARstart{R}{ecently}, there has been a surge of interest in the development of deep learning architectures on network data. Compared to simple graphs that only account for pairwise interactions between two nodes, hypergraph provides a general representation of multi-way relations that are widespread in complex systems~\cite{gao2022hgnn+, hypergraph_review}. For example, human pose estimations with joint parts represented in hypergraph yield coordinated movement~\cite{hao2021hypergraph, liu2020semi}. Recommendation systems with hypergraph modeling co-purchased items lead to a more coherent understanding of user intents~\cite{la2022music, xia2021self}. Gene expression imputation with hypergraphs that depict multi-tissue references imposes superior performance for capturing groups of related genes~\cite{vinas2023hypergraph}. 
\begin{figure}[t!]
  \centering
  \vspace{-0.3cm}
  \includegraphics[width=\linewidth, trim={4.5cm 5.3cm 4cm 7cm}, clip]{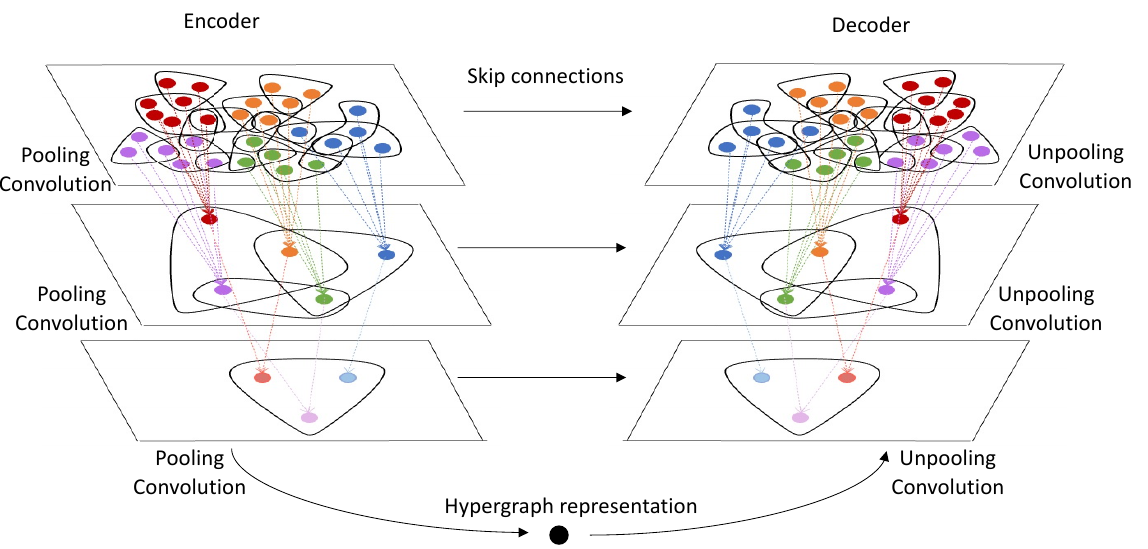}
   \vspace{-0.3cm}
  \caption{Architecture of the proposed HyperGraph U-Net. Each layer in the encoder performs pooling and convolution; the output of the encoder is a ``supernode" representing the overall hypergraph, and it is further fed into a decoder to reconstruct the original hypergraph through convolution and unpooling. To combine high-resolution features from earlier layers, skip connections from corresponding encoder layers are added to decoder layers. Note that the figure ignores signals for simplicity.}
  \label{fig:HGUN}
\end{figure}



With the rise of hypergraph structures, designing deep learning architectures that can effectively process hypergraph data has become a central research focus. Most hypergraph neural networks (HyperGNNs) \cite{HGNN, HCHA, hnhn, huang2021unignn, Allset, wang2024t} follow convolutional or message-passing paradigms, aggregating and transforming features from incident nodes and hyperedges to produce node embeddings. While these representations excel at node-level tasks such as classification, they exhibit critical shortcomings when extended to hypergraph-level objectives that demand a single, global embedding. A flat readout that uniformly pools all node embeddings often washes out vital local and mesoscale structures. For example, summing across communities in a social hypergraph can dilute community-specific signals \cite{ying2018hierarchical, hyperg_isomorphism}, and coarse node-wise pooling in molecular hypergraphs may obscure key bonds and functional substructures. Moreover, even for purely node-level predictions, convolutional HyperGNNs can suffer from an exponentially decaying effective receptive field, which makes deeper stacking of convolutional layers increasingly ineffective~\cite{finder2025improving}. Hierarchical pooling counteracts these issues by coarsening the topology, expanding each node’s receptive field with fewer layers, and preserving multi-scale structure. Empirically, U-Net style graph architectures that interleave pooling and unpooling have achieved state-of-the-art performance on long-range node classification benchmarks, which underscores the value of global context even when supervision is local~\cite{graphunet, vonessen2024next, zhong2023hierarchical}.

Motivated by these insights, we investigate hypergraph pooling both as a means to obtain powerful hypergraph-level embeddings and as a tool to enhance node-level performance. We propose to enhance convolutional HyperGNNs with hierarchical node pooling and unpooling layers, named HyperGraph U-Net (\textbf{HGUN}). As shown in Fig.~\ref{fig:HGUN}, the HGUN follows an encoder-decoder architecture. The encoder takes a hypergraph and its signals as input and then condenses the hypergraph along with node representations, learned by convolution, into a smaller hypergraph. The HGUN encoder outputs a super-node/hypergraph embedding that leverages coarse-grained local structures and can be used to perform hypergraph-level tasks. When node-level embedding is necessary, the HGUN decoder then takes the latent low-dimensional representation of the hypergraph and restores the hypergraph structure through unpooling and convolution. Analogously to the U-Nets in the image~\cite{ronneberger2015u}, the corresponding layer-wise features from the encoder are injected into the decoder through skip connections~\cite{he2016deep}. By varying hypergraph sizes from pooling and unpooling operations, HGUN provides a notion of spatial locality on hypergraphs, amplifies the receptive fields~\cite{ranjan2020asap} of an input hypergraph, leading to better generalization performance~\cite{SEP} and more robust node embeddings.

Specifically, we design the pooling and unpooling operations from a lens of hierarchical clustering on hypergraphs, as the idea of hierarchical node clustering perfectly matches coarse-grained structure exploitation. By a greedy algorithm on selecting the optimal coarse-grained structure from a hierarchical tree, the clustering assignments at different resolution levels are generated all at once, entailing the Parallel Hierarchical Pooling (PHPool) and Unpooling (PHUnpool) operations that bear less risk of information loss. While the proposed PHPool and PHUnpool operations can be combined with any convolutional HyperGNNs, we also introduce a novel HyperGraph CROSS Convolution (HGXConv) layer that models higher-order joint effect through cross-node polynomials. The formulation of HGXConv is motivated by the recent work in tensor HyperGNNs~\cite{wang2024t}, where a cross-node interaction tensor is constructed to model collective node relationships within hyperedges. 

In summary, we highlight our contribution as follows:
\begin{itemize}
    \item We present HyperGraph U-Net (\textbf{HGUN}), a general hypergraph autoencoder framework that is not only effective for node-level tasks but also for hypergraph-level tasks, to the best of our knowledge, the first systematic HyperGNN U-Net framework with explicit hypergraph pooling and unpooling operations.
    \item We introduce Parallel Hierarchical Pooling (\textbf{PHPool}) and Unpooling operations, which select the most promising coarse-grained structures all at once to alleviate topological distortion in sequential processing.
    \item We propose a new HyperGraph Cross Convolution (\textbf{HGXConv}) operation, which propagates node information to hyperedges with cross-node multiplication to capture higher-order mixed effects.
\end{itemize}
The remainder of the paper is organized as follows. We first review related work on hypergraph neural networks in Section II, followed by the preliminaries of hypergraphs in Section III. Then, the proposed hypergraph U-Net architecture is introduced in Section IV, and theoretical analysis is conducted in Section V. Experimental results including simulation and real-world applications are provided in Section VI. We conclude this paper and discuss future work in Section VII.

\section{Related Work}
In this section, we review prior work on hypergraph neural networks (HyperGNNs), hierarchical pooling in U-Nets, and hierarchical clustering on non-Euclidean data.


\subsection{Hypergraph Neural Networks (HyperGNNs)}
A wide variety of HyperGNNs have been proposed in recent years, including methods inspired by hypergraph reductions~\cite{HGNN, HCHA}, two-stage propagation~\cite{hnhn, Allset, huang2021unignn, wu2020comprehensive, gao2022hgnn+} and tensor hypergraph signal processing~\cite{wang2023unified, wang2024t}. Despite different motivations and foundations introduced in these approaches, most of them fall in the framework of convolution~\cite{GCN} or neural message passing~\cite{gilmer2017neural}, where node embeddings are iteratively computed by aggregating node features from their neighborhood. Such a flat aggregation paradigm that keeps node size unchanged has a limitation: it overlooks hierarchical coarse-grained information and so less likely to capture important group-wise properties widely spread in real-world applications~\cite{fortunato2016community}. As a response, we propose pooling and unpooling layers by leveraging hierarchical clustering for hypergraphs and further strengthening cross-node message passing within hyperedge groups. 
\begin{figure}[t!]
  \centering
  \includegraphics[width=0.95\linewidth]{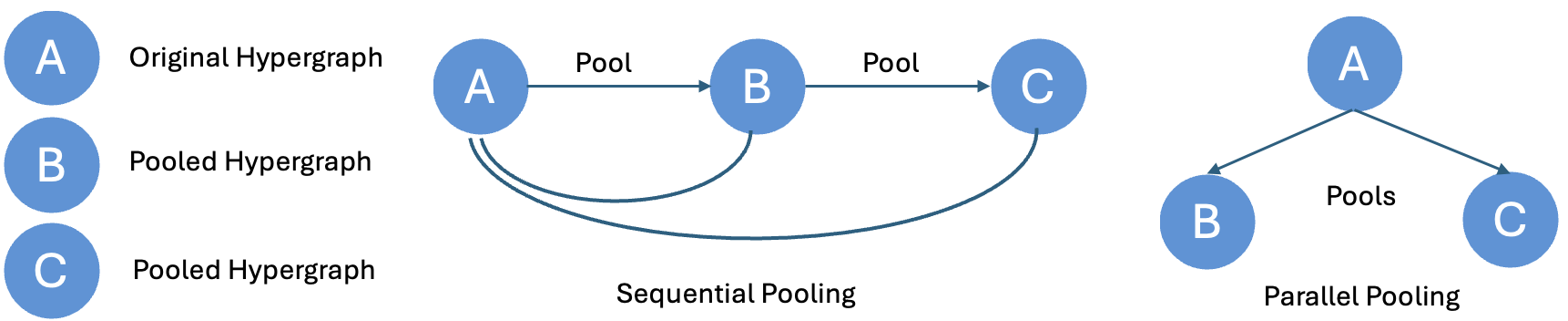}
  \caption{Comparison between sequential pooling and parallel pooling. Parallel pooling preserves at least as much information as sequential pooling under the assumptions of Proposition~\ref{prop:info_loss_compare} in Section~\ref{subsec:mutual_info}.}
  \label{fig:data_processing}
\end{figure}

\subsection{Hierarchical Pooling in Graph U-Nets}
The basic building blocks in graph u-nets are pooling and convolution.  Here, we focus on pooling layers where hierarchical node pooling methods aim to gradually coarsen the network structure into a smaller-sized graph while preserving the original structural information. Hierarchical node pooling can be classified into node drop pooling and node cluster pooling. As suggested by these two names, node drop pooling methods (TopKPool~\cite{graphunet}, SAGPool~\cite{lee2019self}, SODPool~\cite{wang2020second}, etc.) select the most important nodes by learning a score function and discard unselected nodes. Though efficient, node-drop pooling would result in information loss and isolated subgraphs, which could damage graph local structures and cause sparse connectivities, especially after stacking several layers. Thus, a more common design is proposed based on the clustering of nodes. DiffPool~\cite{ying2018hierarchical} learns a cluster assignment (i.e., pooling mapping) from graph neural networks, MinCutPool~\cite{bianchi2020spectral} relaxes the constraints for spectral clustering and designs a clustering-based loss function to enforce ideal clusters, and HoscPool~\cite{duval2022higher} further extends the spectral clustering loss function to edge and triangle motifs. Another line of work~\cite{wang2020haar, eliasof2023haar} draws inspiration from the compressive Haar transform to filter out less important node features. 

Although mitigating information loss, most node drop and node cluster pooling methods employ a sequential pooling procedure, which could excessively compress information as the depth of the neural network increases. With the conceptual difference illustrated in Fig.~\ref{fig:data_processing}, the mutual information between the original hypergraph and the pooled hypergraphs in a sequential data processing pipeline tends to be less than that in a parallel pooling procedure (see Proposition~\ref{prop:info_loss_compare} in subsection~\ref{subsec:mutual_info}). Consequently, we propose a method to simultaneously learn the clustering assignments for all layers by selecting appropriate clusters from a hierarchical clustering dendrogram. Such parallel pooling reduces information loss compared to sequential processing because it captures more of the original data's information by directly deriving multiple outputs from the source without intermediate steps that could degrade the information content. These advantages will be demonstrated in the experiments section.

\subsection{Hierarchical Clustering on Non-Euclidean Data}
A graph hierarchical clustering algorithm takes as input a graph structure, and outputs a hierarchy of clusters represented in a tree or dendrogram as shown in Fig.~\ref{fig:hierarchy_intro} (right). Most hierarchical clustering algorithms are bottom-up methods. They start with the trivial clustering where each node itself forms a cluster. At coarser resolution levels, clusters are gradually merged to create larger clusters and eventually produce the trivial cluster that includes all nodes. In the dendrogram, the horizontal axis denotes the nodes and the vertical axis denotes the distance measured between clusters. For example, in Fig.~\ref{fig:hierarchy_intro}, the corresponding node colors represent the cluster assignment obtained at distance 3, where the purple nodes are merged to form a larger cluster due to the shared triangular structures. A variety of distance measures are introduced in the literature, including the minimum, the maximum, the average distance between any pair of nodes between two clusters, and many others. Most hierarchical clustering research focuses on reducing the time complexity. Recently, some algorithms~\cite{dhulipala2021hierarchical,dhulipala2023terahac} have been shown to obtain linear time complexity with respect to the number of nodes, making hierarchical clustering a practically scalable approach to many problems.

\begin{figure}[t!]
  \centering
  \includegraphics[width=\linewidth]{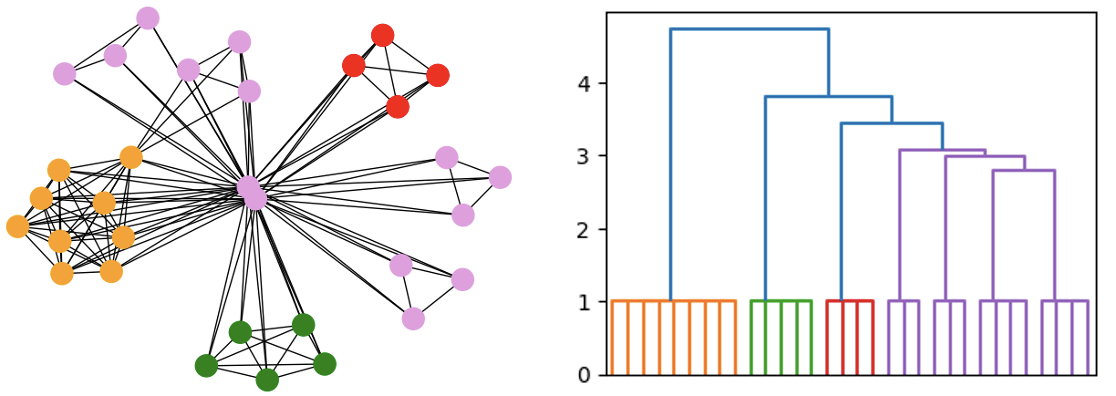}
  \caption{A ego-network graph (left) and its hierarchical clustering dendrogram (right). The node colors correspond to the cluster labels obtained at resolution level $3$, where $4$ clusters (orange, green, red, purple) are formed.}
  \label{fig:hierarchy_intro}
\end{figure}

To the best of our knowledge, hypergraph clustering has only been studied using flat clustering algorithms. Two branches of flat hypergraph clustering are spectral clustering~\cite{zhou2004regularization, CE_learning, li2017inhomogeneous} and modularity-based clustering~\cite{kumar2020new, austin_modularity, feng2023modularity}. The essence of spectral clustering is to learn the spectrum of the hypergraph Laplacian, which might group nodes with isomorphic topology together even if they are not connected. Therefore, we focus on modularity-based clustering, which aims to maximize the density of links inside clusters compared to links between communities, generating more realistic clusters that capture local community structures. 

Different from existing work, we first study hierarchical clustering for hypergraph data, and apply the clustering technique to perform hypergraph pooling operations.

\section{Preliminary on Hypergraphs}

{\bf Notation.} A hypergraph $\mathcal{G}$ is defined as a pair of two sets $\mathcal{G} = (\mathcal{V}, \mathcal{E})$, where $\mathcal{V} = \{v_1, v_2, ..., v_N\}$ denotes the set of $N$ nodes (or vertices) and $\mathcal{E} = \{e_1, e_2, ..., e_K\}$ is the set of $K$ hyperedges whose elements $e_k$ ($k = 1, 2,...,K$) are non-empty subsets of $\mathcal{V}$. The incidence matrix of a hypergraph $\mathbf{H}\in \{0, 1\}^{N\times K}$ can be used to represent the hypergraph topology, where $[\mathbf{H}]_{ve} = 1$ if node $v$ is contained in the hyperedge $e$ and zero otherwise. Apart from the hypergraph structure, there are also features $\mathbf{f}_v \in \mathbb{R}^D$ associated with each node $v\in\mathcal V$ and features  $\mathbf{f}_e \in \mathbb{R}^F$ for each hyperedge $e \in \mathcal{E}$, which are used as row vectors to construct the vertex feature matrix $\mathbf{F}_{\mathcal{V}}\in \mathbb{R}^{N\times D}$, and hyperedge feature matrix $\mathbf{F}_{\mathcal{E}}\in \mathbb{R}^{K\times F}$. The degree of a node $d(v)$ in a hypergraph is defined as the number of hyperedges that include the node. Similarly, the degree of a hyperedge $d(e)$ in a hypergraph is defined as the number of nodes included by the hyperedge.
The degrees are contained in the diagonal node degree matrix $\mathbf{D}_{\mathcal{V}} \in \mathbb{R}^{N\times N}$ and hyperedge degree matrix $\mathbf{D}_{\mathcal{E}} \in \mathbb{R}^{K\times K}$. 

{\bf Modularity.} We use hypergraph modularity to guide our selection of clusters from a hierarchical dendrogram. The idea of modularity is to compare the number of connections within the cluster with its expected value from a null model. Higher modularity indicates more dense connections within clusters, which is considered a desirable clustering assignment. As established in work~\cite{kumar2020new}, the node-degree-preserving adjacency matrix $\mathbf{A}^{hyp} = \mathbf{H}(\mathbf{D}_{\mathcal{E}} - \mathbf{I})^{-1}\mathbf{H}^{T}$ is proposed to compute hypergraph modularity. The reason for defining the adjacency matrix with a scaling factor $(d(e) - 1)$ is to preserve node degrees as in the original hypergraph.

The node connections underlying in a null model are denoted as $[\mathbf{P}^{hyp}]_{uv} = \frac{d(u) d(v)}{d_\mathcal{V}}$, with $d_\mathcal{V} = \sum_{u\in \mathcal{V}} d(u) $ being the sum of node degrees, which computes how likely two nodes $u$ and $v$ were connected if we randomly construct edges between them proportional to the degree of nodes. The expression of hypergraph modularity is then given by
\begin{equation}\label{eq:mod}
    Q^{hyp}(\mathbf{H}, \mathbf{S}) = \frac{1}{d_{\mathcal{V}}} \sum_{uv} ([\mathbf{A}^{hyp}]_{uv} -[\mathbf{P}^{hyp}]_{uv}) \delta(s_u, s_v),
\end{equation}
where \(\mathbf{S}\in \mathbb{R}^{N \times n_{\text{cluster}}}\) is a binary cluster assignment matrix, and \(s_u\) denote the assigned cluster of node \(u\). 
 $\delta(s_u, s_v)$ is an indicator variable, $\delta(s_u, s_v)=1$ if nodes $u$ and $v$ are assigned in a same cluster and zero otherwise. If two nodes are in the same cluster and are connected, the difference $[\mathbf{A}^{hyp}]_{uv} -[\mathbf{P}^{hyp}]_{uv}$ will be positive, contributing to higher modularity, and vice versa. Consequently, a higher modularity means that there are more connections within clusters and thus fewer connections between clusters as the number of connections is fixed.  The notations used in this paper are summarized in Table~\ref{table:notation}.

{\bf Problem formulation.} The goal of hypergraph neural networks is to learn a mapping $f: (\mathcal{G}, \mathbf{F}_{\mathcal{V}}) \to \mathbf{Y} \in \mathbb{R}^{M\times c}$, where $c$ denotes the number of label classes. The value of $M$ can be the number of nodes for the node-level task or just $1$ for the hypergraph-level task. From an encoder-decoder perspective, the HGUN has the following signature:
\begin{equation}
    \begin{cases}
        &\text{Enc}: (\mathcal{G}, \mathbf{F}_{\mathcal{V}}) \to (\mathcal{G}^P,\mathbf{F}_{\mathcal{V}}^P),  \\
        &\text{Dec}: (\mathcal{G}^P,\mathbf{F}_{\mathcal{V}}^P) \to (\mathcal{G}, \mathbf{Y}), \text{where } \mathbf{Y} \in \mathbb{R}^{M\times c},
    \end{cases}
\end{equation}
where the encoder compresses the original hypergraph and node features to a latent space such that the pooled hypergraph $\mathcal{G}^P$ has less number of nodes $N_P < N$. Note that the condensed hypergraph $\mathcal{G}^P$ does not necessarily have fewer hyperedges because the hyperedges in the pooled hypergraph are constructed by exploring higher-order connections among nodes (i.e. cliques), which are data dependent. The pooled hypergraph is associated with its pooled node signals $\mathbf{F}_{\mathcal{V}}^P\in \mathbb{R}^{N_P \times h}$ with hidden dimension $h$. Taking the tuple of $(\mathcal{G}^P, \mathbf{F}_{\mathcal{V}}^P)$ as input, the decoder then recovers the original hypergraph $\mathcal{G}$ while constructing new output features $\mathbf{Y}$. 

We define the number of layers in HGUN as the number of HGXConv $+$ PHPool/ PH-UnPool blocks. For example, the HGUN shown in Fig.~\ref{fig:HGUN} consists of $3$ layers. Note that the hypergraph topologies in HGUN encoder and decoder are identical at the same layer, which acts as a backbone to support the reconstruction of the output features~\cite{SEP, graphunet}. Allowing for different hypergraph structures in encoder and decoder would transition our model into a generative framework, which is beyond the scope of this paper. For further exploration of models that adapt and refine hypergraph structures within neural networks, we direct interested readers to the work described in~\cite{hypergraph_structure_learning}.

\begin{table}[b!]
    \centering
    \renewcommand{\arraystretch}{1.2} 
    \caption{Notation Table}
    \begin{tabular}{c|l}
        \hline
        \textbf{Symbol} & \textbf{Description} \\
        \hline
        $\mathcal{G} $ & Hypergraph structure $\mathcal{G}$ \\
        $\mathcal{V}$ & Vertex set of a hypergraph \\
        $\mathcal{E}$ & Hyperedge set of a hypergraph \\
        $N$ & Number of nodes in a hypergraph, i.e., $|\mathcal{V}|$ \\
        $K$ & Number of hyperedges in a hypergraph, i.e., $|\mathcal{E}|$  \\
        $\mathbf{H}$ & Hypergraph incidence matrix \\
        $\mathbf{A}$ & Hypergraph adjacency matrix, $\mathbf{HH}^T$\\
        $\mathbf{A}^{hyp}$ & Node-degree-preserving adjacency matrix, $\mathbf{H}(\mathbf{D}_{\mathcal{E}} - \mathbf{I}) \mathbf{H}^T$\\
        $\mathbf{D}_{\mathcal{V}}$ & Diagonal node degree matrix\\
        $\mathbf{D}_{\mathcal{E}}$ & Diagonal hyperedge degree matrix\\
        $d_{\mathcal{V}}$ & Sum of node degrees\\
        $(\cdot)^{(l)}$ & Variable at the $l^{\rm th}$ layer of hypergraph U-Nets\\
        $\mathbf{F}_{\mathcal{V}}$ & Node feature matrix\\
        $\mathbf{F}_{\mathcal{E}}$ & Hyperedge feature matrix\\
        $\mathbf{Y}$ & Output embedding for downstream tasks\\
        $\mathbf{f}_v$ & Feature vector associated with node $v$\\
        $\mathbf{f}_e$ & Embedding vector for hyperedge $e$ \\
        $\mathbf{S}$ & Clustering assignment matrix for node pooling\\
        $\mathbf{W}$ & Learnable weight matrix in a neural network\\
        $\mathbf{b}$ & Learnable bias matrix in a neural network\\
        $\sigma$ & Activation function for a neural network\\
        \hline
    \end{tabular}\label{table:notation}
\end{table}

\section{Hypergraph U-Nets}
We introduce the HyperGraph U-Net (HGUN) architecture, which consists of three major components: Parallel Hierarchical Pooling (PHPool), Unpooling (PHUnpool), and HyperGraph Cross-Convolution (HGXConv). 
\begin{figure*}[ht]
  \centering
  \includegraphics[width=\linewidth]{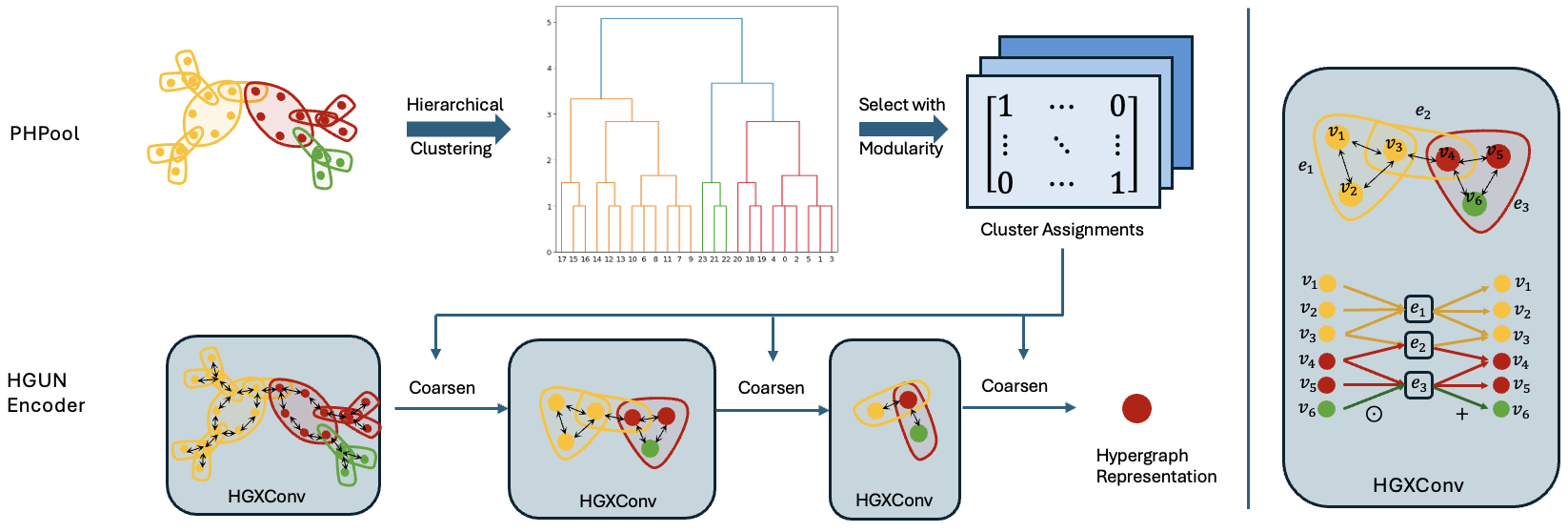}
  \caption{The overall architecture of the proposed PHPool operator combined with HGXConv to formulate the HGUN Encoder. The clustering assignments of PHPool come from a flattened hierarchical dendrogram and are injected to perform coarsening after each HGXConv layer. The two-stage propagation rule of HGXConv is displayed on the right.}
  \label{fig:overall_architecture}
\end{figure*}
Without losing generality, we present the overall HGUN encoder architecture in Fig.~\ref{fig:overall_architecture} as the decoder is constructed by replacing PHPool with PHUnpool to enlarge a hypergraph. As shown in Fig.~\ref{fig:overall_architecture}, PHPool first uses a hierarchical clustering algorithm that takes a hypergraph as input and outputs a hierarchical dendrogram, capturing coarse-grained structures of hypergraphs. To obtain reasonable and meaningful cluster assignments, we greedily select clusters from the dendrogram under the guidance of modularity. These selected cluster assignments are injected into the HGUN encoder to coarsen corresponding hypergraphs. The detailed description of PHPool is given in subsection~\ref{subsec:phpool}. The HGXConv, furthermore, aggregates neighboring information by modeling higher-order interactions via cross-node multiplication, which is presented in subsection~\ref{subsec:HGXConv}.

\subsection{Parallel Hierarchical Pooling (PHPool)}\label{subsec:phpool}

Consider node-pooling as a node-grouping problem, where each pooling layer maps a hypergraph into a condensed hypergraph with fewer nodes. Let $ \mathbf{S}^{(l)}\in \{0, 1\}^{N_l \times N_{l+1}}$ be a binary cluster assignment matrix that represents the pooling relationship between two consecutive layers, where $N_l > N_{l+1}$. Formally, the pooling procedure can be described as two steps: 
\begin{equation}
    \begin{cases}
        \text{Cluster Assignments: } \mathcal{G} \to \mathcal{S}=\{\mathbf{S}^{(1)}, ..., \mathbf{S}^{(L)}\};\\
         \text{Coarsening: } (\mathcal{G}^{(l)}, \mathbf{F}_{\mathcal{V}}^{(l)};\mathbf{S}^{(l+1)}) \to (\mathcal{G}^{(l+1)}, \mathbf{F}_{\mathcal{V}}^{(l+1)}),
    \end{cases}
\end{equation}
where $l=0, 1, ..., L$ and $(\mathcal{G}^{(0)}, \mathbf{F}_{\mathcal{V}}^{(0)})$ is the original hypergraph and features. The first step takes a hypergraph as input and generates a series of clustering assignment matrices $\mathbf{S}^{(l)}$ that map the nodes between two consecutive layers. $[\mathbf{S}^{(l)}]_{ij}$ denotes whether the $i$-th node in the current hypergraph belongs to the $j$-th cluster, that is, the $j$-th node in the pooled hypergraph. The coarsening step can be seen as reducing the hypergraph tuple $(\mathcal{G}^{(l)}, \mathbf{F}_{\mathcal{V}}^{(l)})$ to a condensed hypergraph with fewer nodes. The cluster assignment step is achieved all at once and thus can be separated as a pre-processing step to speed up the training procedure. We describe these two steps in the following subsections.

\begin{figure*}[ht]
  \centering
  \includegraphics[width=\linewidth]{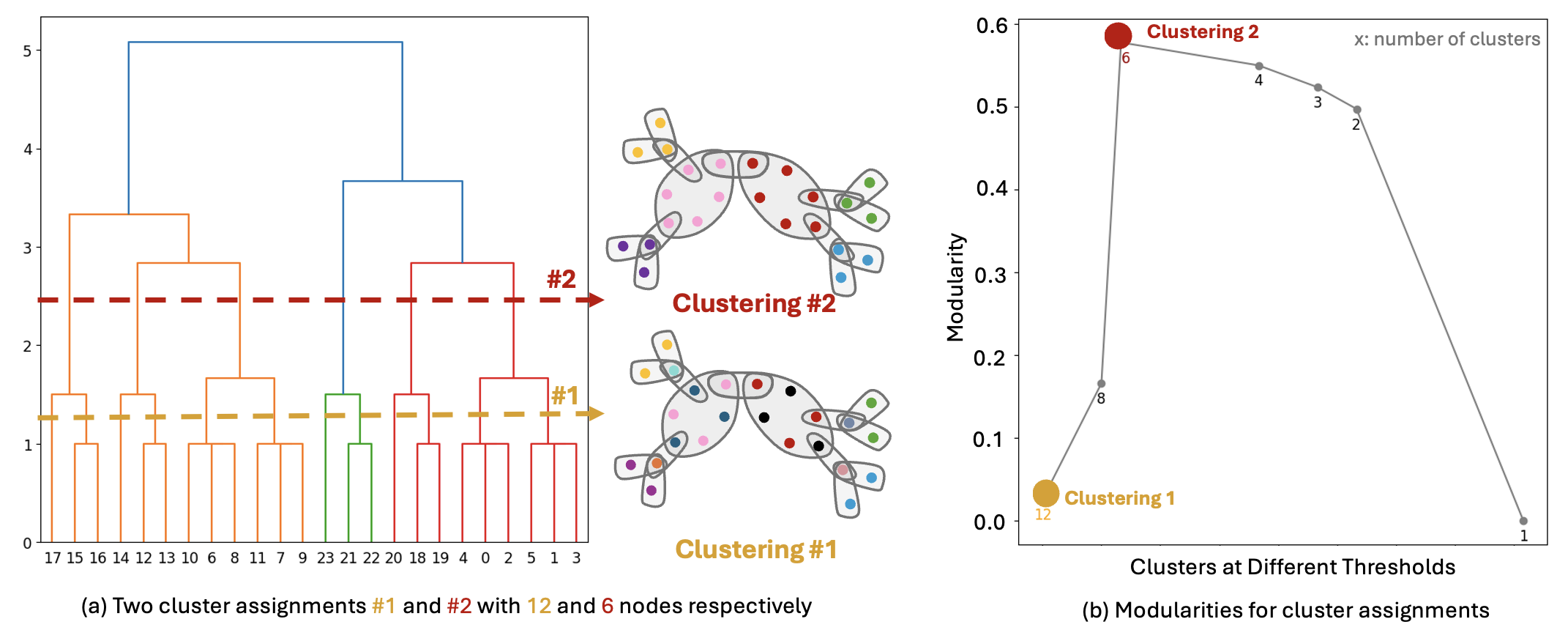}
  \vspace{-0.7cm}
  \caption{(a) Two resulting cluster assignments from the hierarchical dendrogram for a molecule hypergraph. Node colors indicate cluster assignments. (b) Modularities for all unique clustering assignments obtained from the dendrogram.}
  \label{fig:hier_pool}
\end{figure*}
\subsubsection{Cluster Assignments Generation}
\label{subsec:clustering}
As illustrated in Fig.~\ref{fig:overall_architecture} (top), a hypergraph is first fed into the hierarchical algorithm~\cite{mullner2011modern, bar2001fast} to produce a dendrogram that represents coarse-grained structures of a hypergraph from bottom to top. Among many cluster assignments that can be retrieved from the hierarchical dendrogram, we greedily select the optimal granularity thresholds based on hypergraph modularity and generate a series of cluster assignments. The detailed procedure is described as follows.

{\bf (a) Fit hierarchical dendrogram for hypergraphs.}
To apply hierarchical clustering to hypergraphs, we first transform the hypergraph into a metric space by computing shortest-path distances between all pairs of nodes using Dijkstra’s algorithm~\cite{dijkstra2022note}. The resulting distance matrix serves as input to an agglomerative clustering procedure, which iteratively merges nodes and clusters to produce a dendrogram (Fig.~\ref{fig:overall_architecture}). This dendrogram compactly represents the nested community structure of the hypergraph across multiple scales.
Each horizontal cut of the dendrogram corresponds to a particular resolution of the hypergraph, yielding a clustering assignment of the nodes. By varying the cut threshold, we obtain a family of progressively coarser partitions. These multiresolution clusterings form the basis of our pooling and unpooling operators: finer cuts preserve detailed local structures, while coarser cuts emphasize broader community organization.

{\bf (b) Hypergraph modularity-guided threshold selection.} With the agglomerative dendrogram, we then need to choose $L$ cluster assignments, which correspond to $L$ pooling (or unpooling) layers. A straightforward way to select cluster assignments is by specifying the number of clusters, which is equivalent to defining a pooling ratio. For example, if a pooling ratio is $0.5$ for $L=3$ layers, the number of nodes will be $[N, 0.5N, 0.25N, 1]$. This method works well on large hypergraphs that exhibit dense and wide community structures but could be problematic on small hypergraphs such as molecules. For example, the molecule in Fig.~\ref{fig:hier_pool} contains $24$ nodes, and its two cluster assignments (\textcolor{DarkYellow}{\#1} and \textcolor{DarkRed}{\#2}) yield \textcolor{DarkYellow}{$12$} and \textcolor{DarkRed}{$6$} nodes, respectively. Clustering \textcolor{DarkRed}{\#2} preserves important local bond structures, whereas cluster assignment \textcolor{DarkYellow}{\#1} is overly detailed, resulting in topological distortion and less meaningful structures. To quantify clustering performance and select optimal thresholds for a hypergraph, we use hypergraph modularity~\cite{kumar2020new} as the evaluation metric. Hypergraph modularity measures the density of connections within clusters, favoring assignments with more connected hyperedges within clusters. As shown in Fig.~\ref{fig:hier_pool} (b), the curve describes the modularity for all unique cluster assignments obtained from the hierarchical dendrogram. Consistent with expectations, cluster assignment  \textcolor{DarkRed}{\#2} demonstrates higher modularity than clustering \textcolor{DarkYellow}{\#1}, validating the effectiveness of modularity as a clustering metric for selecting meaningful cluster assignments.

Embracing modularity, we propose a greedy algorithm (Algorithm~\ref{algo: flat}) to select the optimal cluster assignments from the constructed hierarchical dendrogram. After evaluating all clustering assignments using modularity, we iteratively select the cluster assignments with the maximum modularity until $L$ grouping layers (ie $L-1$ cluster assignments) are generated. Note that we only need $L-1$ cluster assignments for $L$ pooling layers because the last pooling layer is as trivial as clustering all nodes into a single cluster to generate a holistic hypergraph representation. As elaborated in Algorithm~\ref{algo: flat}, it takes input of the incidence matrix $\mathbf{H}$ of a hypergraph, the linkage matrix $\mathbf{Z}$ and the number of desired grouping layers $L$, and produces a list of cluster assignments $C^{*}$ representing the optimal granularities for an underlying hypergraph.

The algorithm proceeds as follows. The third column of the linkage matrix, denoted by $\mathbf{Z}[:, 3]$, contains the distance measures for all cluster merges. We begin by retrieving the distinct distances $T = \mathbf{Z}[:, 3]$. In the first step, we iterate through these distances, flatten the dendrogram to obtain cluster labels $c(t)$, calculate the corresponding modularity via Eq.~\eqref{eq:mod}, and store the results in a list $Q$. The second stage optimizes the thresholds by selecting those that maximize modularity. Specifically, a greedy approach is employed to first identify the threshold with the highest modularity, then perform a rightward search, ignoring prior modularities. The rightward search is preferred because it constructs hypergraphs from fine-grained to coarse-grained levels, as only this direction produces coarser hypergraphs. Additionally, empirical observations, such as the modularity curve in Fig.~\ref{fig:hier_pool} (right), show that modularity is typically a concave function, with the right side of the peak being smoother than the left. 

The greedy selection process operates in a loop, terminating when either the layer counter $l$ reaches the desired number of layers $L$ or when all distance measures are exhausted. In each iteration, the index of the optimal threshold is identified by maximizing the modularities to the right of the current cursor (i.e., $idx^{*}\!\leftarrow\!\argmax_{j\geq idx} Q[j]$). The modularities up to and including this selected index $idx^*$ are then set to negative infinity, and the cursor is advanced to $idx^{*}+1$ so that the next search continues strictly to the right. The detailed procedure is described in Step 2 of Algorithm~\ref{algo: flat}.

\begin{algorithm}
\caption{Determine Optimal Thresholds via Modularity Maximization}\label{algo: flat} 
\textbf{Input:} The incidence matrix $\mathbf{H}$; The linkage matrix $\mathbf{Z}$ representing a dendrogram; the number of layers $L$; \\
\textbf{Output:} A list of optimal clusters $C^{*}$ for a dendrogram.
\begin{algorithmic}[1]
\State Retrieve all unique distance measure $T = \mathtt{unique}(\mathbf{Z}[:, 3])$ 
\State Initialize an empty list $Q$ to hold modularities
\State //\textit{Step 1: Compute modularities}
\For{each threshold $t$}
    \State Obtain cluster labels $c(t)\in \mathbb{R}^N$ by cutting the dendrogram $\mathbf{Z}$ at threshold $t$
    \State $Q \leftarrow$ $\mathtt{Add}(\mathtt{Modularity}(\mathbf{H}, c(t)))$ according to Eq.~\eqref{eq:mod}
\EndFor
\State //\textit{Step 2: Choose optimal clusters based on modularities }
\State Declare an empty list $C^{*}$ for the optimal clusters
\State Initialize layer counter $l=0$
\State Initialize counter $idx =0$ for iterating distance measure
\State {\bf While} $l < L$ {\bf do}
\State \quad  {\bf If} $idx>=\mathtt{length}(T)$ {\bf do}
\State \quad\quad break
\State \quad Select the optimal threshold index $idx^{*}\!\leftarrow \argmax_{j\geq idx} Q[j]$
\State \quad Add selected cluster $C^{*}\leftarrow \mathtt{Add}(c(T[idx^{*}]))$
\State \quad Mute the prior modularities $Q[:idx^{*}+1]\leftarrow -\infty$
\State \quad Update index $idx \leftarrow idx^{*}+1$
\State  \quad  $l = l+ 1$
\end{algorithmic}
\end{algorithm}

The output cluster assignments $C^*$ from Algorithm~\ref{algo: flat} are labeled based on the original hypergraph, where each cluster label in $C^*$ is a length-$N$ vector, indicating clusters at different resolution levels. To perform sequential coarsening as shown in Fig.~\ref{fig:overall_architecture} (left), mappings between two consecutive layers are needed. We then iteratively examine parent-child relationships among clusters. For each cluster in the $l$-th layer, we identify its members, check the uniformity of their parental cluster assignments from the previous layer, and record a tuple (parent id, child id) that defines the mapping between consecutive granularity levels.




\subsubsection{Higher-order Relationships Augmented Coarsen}

\par 
Note that the tuple (parent id, child id) obtained from PH-Pool above is a sparse representation of the cluster assignment matrix $\mathbf{S}^{(l)} \in \mathbb{R}^{N_l \times N_{l+1}}$. For consistency, we use notation $\mathbf{S}^{(l)}$ to represent the hierarchical mapping between any two consecutive blocks. A new feature matrix and hypergraph structure are then generated by a higher-order relationship augmented coarsening. In particular, for $l=0, 1,...,L$, we apply the following steps
\begin{subequations}\label{eq:PHPool}
    \begin{align}
    &\mathbf{F}_{\mathcal{V}}^{(l+1)} = \mathbf{S}^{(l)^T} \mathbf{F}_{\mathcal{V}}^{(l)} \in \mathbb{R}^{N_{l+1}\times h}, N_{l+1} < N_{l};\label{eq:node_coarse}\\
    &\mathbf{A}^{(l+1)} = \mathbf{S}^{(l)^T} \mathbf{A}^{(l)} \mathbf{S}^{(l)}\in \mathbb{R}^{N_{l+1}\times N_{l+1}}; \label{eq:adj_coarse}\\
   & \mathcal{G}^{(l+1)} = \mathtt{FindCliques}(\mathbf{A}^{(l+1)}).\label{eq:hyperg_coarse}
\end{align}
\end{subequations}
Eq.~\eqref{eq:node_coarse} generates new node embeddings by aggregating current node embeddings in the same cluster.
Next, taking the hypergraph adjacency matrix representation, Eq.~\eqref{eq:adj_coarse} connects new nodes (clusters) if their parent nodes are connected in the previous layer, Eq.~\eqref{eq:hyperg_coarse} further constructs hyperedges using cliques in the coarsened hypergraph adjacency matrix. Eq.~\eqref{eq:adj_coarse} and Eq.~\eqref{eq:hyperg_coarse} together connect the clusters that have connected nodes and raise such connection to hyperedges if a group of nodes are mutually connected. The function $\mathtt{FindCliques}$ is implemented using the clique-finding routine~\cite{cliquealgo} in the Python NetworkX library.

\subsection{Hypergraph Cross Convolution (HGXConv)}\label{subsec:HGXConv}

In addition to the pooling layer, we propose using a hypergraph cross-convolution (HGXConv) operation to encode higher-order interactions among nodes by cross-node products~\cite{wang2024t}. The HGXConv is designed under the neural message passing paradigm~\cite{gilmer2017neural}, aggregating neighboring messages to individual nodes. Consequently, we switch to the node-wise and hyperedge-wise notation, in which $\mathbf{f}_v \in \mathbb{R}^h$ denotes the embedding vector of the node $v \in \mathcal{V}$ and $\mathbf{f}_e \in \mathbb{R}^h$ represents the hyperedge embedding vector for any hyperedge $e\in \mathcal{E}$. 

In the $l$-th layer ($l=0, 1, ..., L-1$), let $(\mathcal{G}^{(l)}, \mathbf{F}_{\mathcal{V}}^{(l)})$ be the hypergraph tuple. Suppose that the hypergraph topology $\mathcal{G}^{(l)}$ consists of a node set $\mathcal{V}^{(l)}$ and a hyperedge set $ \mathcal{E}^{(l)}$. Denote $V^{(l)}(e) = \{v \in \mathcal{V}^{(l)} | v \in e \}$ as the set of vertices that are contained in hyperedge $e$, and $E^{(l)}(v) = \{e\in \mathcal{E}^{(l)} | e \ni v\} $ as the hyperedge set that contains node $v$, one layer of the HGXConv is composed of three modules:
\begin{equation}
    \begin{cases}
        \mathtt{Node2Hyperedge}&: \mathbf{f}_e^{(l+1)} = f_{\mathcal{V}\to \mathcal{E}} (\mathbf{f}_v^{(l)}, V^{(l)}(e));\\ 
         \mathtt{Hyperedge2Node}&: \mathbf{f}_{\mathcal{N}(v)}^{(l+1)} = f_{\mathcal{E}\to \mathcal{V}} (\mathbf{f}_e^{(l+1)}, E^{(l)}(v));\\
        \qquad \mathtt{Fusion}&: \mathbf{f}_v^{(l+1)} = f_{\mathtt{combine}}(\mathbf{f}_v^{(l)},\mathbf{f}_{\mathcal{N}(v)}^{(l+1)}).
    \end{cases}
\end{equation}
The operation $\mathtt{Node2Hyperedge}$ aggregates information from individual nodes to their corresponding hyperedges as shown in Fig.~\ref{fig:overall_architecture} (right), generating an intermediate hyperedge embedding $\mathbf{f}_e$, differentiating hypergraph convolution from regular graph convolution~\cite{hnhn, Allset}. Following this, $\mathtt{Node2Hyperedge}$ processes these embedded hyperedge representations back to the nodes that are connected to each hyperedge. The resulting node embedding $\mathbf{f}_{\mathcal{N}(v)}^{(l+1)} \in \mathbb{R}^{h}$ is referred to the neighborhood embedding for node $v$, because this embedding captures the summarized local connectivity patterns around each node. Considering that neighboring nodes do not necessarily behave similarly to central nodes, it is advantageous to incorporate the original node features again. To address this, we introduce a fusion module that combines the original node features with the embedded neighborhood features from message passing. 

We describe the three modules of HGXConv using layer-agnostic notation in the following subsections for clarity. However, it is important to note that both the input hypergraph tuple data and the resulting embeddings differ across layers.

\subsubsection{Node2Hyperedge}
\par
The goal of $\mathtt{Node2Hyperedge}$ is to obtain hyperedge embeddings that capture the mixed (higher-order) effect among the nodes incident to a hyperedge. Inspired by the tensor hypergraph message passing (T-HGMP)~\cite{wang2024t}, we define the $\mathtt{Node2Hyperedge}$ mapping via element-wise multiplicative mixing. 
Specifically, consider a hyperedge $e$ with incident node set $V(e)=\{v_1,\ldots,v_k\}$, where $k=|V(e)|$ denotes the hyperedge degree. Let $\mathbf{f}_{v_i}\in\mathbb{R}^{h}$ be the feature vector of node $v_i$. {We first compute a transformed node representation
\begin{equation}\label{eq:n2e_node_transform}
\mathbf{a}_{v_i} = \sigma_1\!\left(\mathbf{W}_v \mathbf{f}_{v_i} + \mathbf{b}_v\right)\in\mathbb{R}^{h'},
\end{equation}
where $\mathbf{W}_v\in\mathbb{R}^{h'\times h}$ and $\mathbf{b}_v\in\mathbb{R}^{h'}$ are learnable parameters shared across nodes. We then define the node-to-hyperedge aggregation $f_{\mathcal{V}\to \mathcal{E}}:\mathbb{R}^{k\times h}\to\mathbb{R}^{h'}$ as the element-wise geometric mean:
\begin{equation}\label{eq:cross_node}
\mathbf{f}_e
=
\left(
\mathbf{a}_{v_1}\odot \cdots \odot \mathbf{a}_{v_k}
\right)^{\odot \frac{1}{k}},
\end{equation}
where $\odot$ denotes element-wise product and $(\cdot)^{\odot \frac{1}{k}}$ denotes the element-wise power. Eq.~\eqref{eq:cross_node} is permutation invariant with respect to the node order in $V(e)$, and can be viewed as a degree-normalized $k$-way multiplicative feature mixing operator aiming to capture higher-order node interactions. For implementation, we use $\sigma_1(\cdot)=\mathrm{softplus}(\cdot)=\log(1+e^{(\cdot)})$ to ensure positivity. To improve numerical stability, we compute Eq.~\eqref{eq:cross_node} in the log-domain:
\begin{equation}\label{eq:cross_node_log}
\mathbf{f}_e
=
\exp\!\left(
\frac{1}{k}\sum_{i=1}^{k}\log\!\left(\mathbf{a}_{v_i}\right)
\right),
\end{equation}
where $\log(\cdot)$ and $\exp(\cdot)$ are applied element-wise.

\noindent {\bf Remark.} Eq.~\eqref{eq:cross_node} can be interpreted as inducing rich cross-node interactions. 
To build intuition, consider a scalar feature per node and ignore nonlinearities. Then the aggregation reduces to
$(w f_{v_1} + b)(w f_{v_2} + b)\cdots (w f_{v_k} + b)$, whose expansion contains not only the pure $k$-way product
$w^k \prod_{i=1}^k f_{v_i}$ but also many lower-order cross terms due to the bias term. 
This illustrates that applying an affine transform before multiplicative mixing can capture interaction patterns beyond a single fixed order (up to order $k$), providing a more expressive alternative to purely additive aggregation.}

\subsubsection{Hyperedge2Node} To aggregate hyperedge embeddings back to nodes, we simply define $f_{\mathcal{E}\to \mathcal{V}}$ as a multilayer-perceptron (MLP) followed by summation:
\begin{equation}\label{eq:edge_to_node}
    \mathbf{f}_{\mathcal{N}(v)} = \sum_{e\in E(v)} \sigma_3 (\mathbf{W}_e \mathbf{f}_e + \mathbf{b}_e),
\end{equation}
where $\mathbf{W}_e\in \mathbb{R}^{h'\times h}$ and $\mathbf{b}_e\in \mathbb{R}^{h'}$ are trainable parameters in the neural network and ReLU is used for $\sigma_3$ in practice. 

\subsubsection{Fusion} Inspired by GraphSAGE~\cite{GraphSage}, we propose to combine self node feature and neighborhood embedding and define $f_{\mathtt{combine}}$ as: 
\begin{equation}\label{eq:fusion}
    \mathbf{f}'_v = \sigma_4(\mathbf{W}_{\mathtt{fuse}}[\mathbf{f}_v, \mathbf{f}_{\mathcal{N}(v)}]  + \mathbf{b}_{\mathtt{fuse}}).
\end{equation}
The concatenated feature $[\mathbf{f}_v, \mathbf{f}_{\mathcal{N}(v)}]\in \mathbb{R}^{2h}$ is first fed into an MLP with learnable weight matrix $\mathbf{W}_{\mathtt{fuse}} \in \mathbb{R}^{h\times 2h}$ and bias matrix $\mathbf{b}\in \mathbb{R}^{h}$, and then passed in the ReLU activation function.  This concatenation allows the MLP to learn the importance of weights assigned to the self-node feature and the neighborhood feature.

\subsection{Hypergraph U-Net Architecture}
The hypergraph u-net (HGUN) architecture is illustrated in Fig.~\ref{fig:HGUN}. It consists of an encoder and a decoder to coarsen and expand hypergraphs along with node features. The encoder contains several blocks, each of which is constructed by an HGXConv layer and a PHPool layer. HGXConv aggregates information from the local neighborhood to hyperedges and back to each node, and PHPool is responsible for reducing the size of hypergraphs to consider clustering characteristics according to Eq.~\eqref{eq:PHPool}. Note that the cluster assignment generation is determined in a pre-processing step and injected into HGUN's for coarsening as illustrated in Fig.~\ref{fig:overall_architecture} (left). This design deviates from end-to-end training, but the benefit of doing so is to construct modularity-guided pooling operations across layers (i.e., the sum of hypergraph modularity is maximized along a globally constructed dendrogram cut), while sequential processing only solves the problem for layer-wise local optimization. In subsection~\ref{subsec: depth_sensitivi}, we will show that this pooling design is robust to the number of layers.

The proposed HGUN can be used in encoder-only or encoder-decoder formats as shown in Fig.~\ref{fig:two_variation}. Using only the encoder, we can obtain a hypergraph embedding and perform hypergraph classification tasks (subsection~\ref{subsec:hypergraph_classification}). On the other hand, attaching a decoder to the condensed embedding can recover the original hypergraph structure and learn to obtain new node features according to the task. Particularly, the HGUN decoder is mirrored to the encoder, except we conduct PH-UnPool by reversing the process in Eq.~\eqref{eq:PHPool} as follows:
\begin{subequations}\label{eq:PH-UnPool}
    \begin{align}
    &\mathbf{F}_{\mathcal{V}_{dec}}^{(l-1)} = \mathbf{S}^{(l)} \mathbf{F}_{\mathcal{V}_{dec}}^{(l)} \in \mathbb{R}^{N_{l-1}\times h}, N_{l-1} > N_{l};\\
    &\mathbf{A}_{dec}^{(l-1)} = \mathbf{S}^{(l)} \mathbf{A}_{dec}^{(l)}   \mathbf{S}^{(l)^T} \in \mathbb{R}^{N_{l-1} \times N_{l-1}};\\
    &\mathcal{G}_{dec}^{(l-1)} = \mathtt{FindClique}(\mathbf{A}_{dec}^{(l-1)} ),
\end{align}
\end{subequations}
for $l = L,..., 1$. The starting point of the HGUN decoder is the condensed hypergraph structure and node features obtained from encoder, and the eventual output of HGUN decoder is $(\mathcal{G}, \mathbf{Y}) = (\mathcal{G}_{dec}^{(0)}, \mathbf{F}_{\mathcal{V}_{dec}}^{(0)})$. This encoder-decoder architecture thus can be used in node signal reconstruction (subsection~\ref{subsec:synthetic}) and classification tasks (subsection~\ref{subsec:anomaly}).

\begin{figure}[t!]
  \centering
  \includegraphics[width=\linewidth]{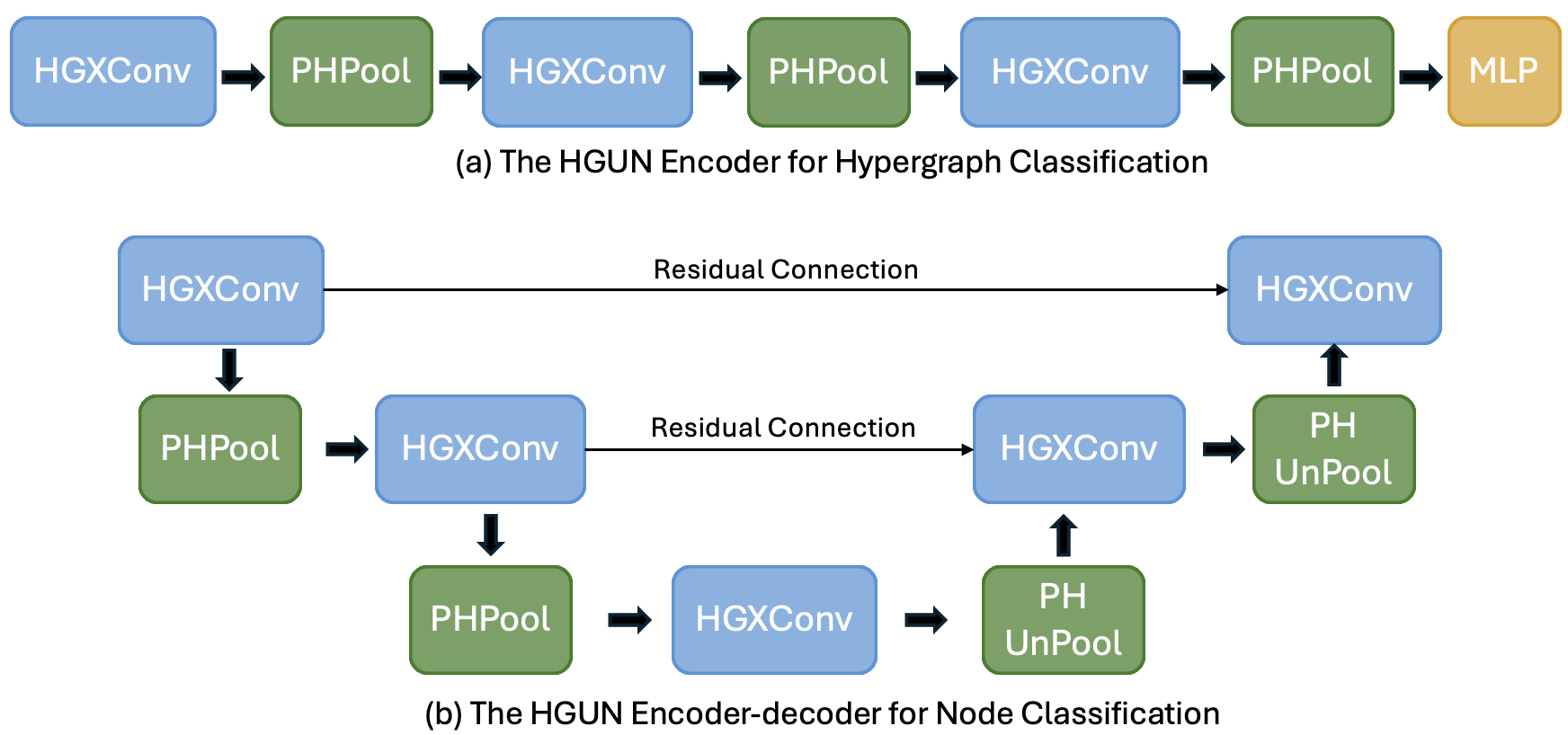}
  \vspace{-0.7cm}
  \caption{Two variations of HGUN: (a) Encoder only for hypergraph-level tasks; (b) Encoder-decoder for node-level tasks.}
  \label{fig:two_variation}
\end{figure}


\section{Theoretical Analysis}
The proposed hypergraph u-net includes two major components: the HGXConv layer and the PHPool layer; we theoretically analyze the computational complexity and show the desired properties of these two components.
\subsection{Computational Complexity Analysis}


\begin{proposition}
The time complexity of building $L$-layer PHPool operators is $\mathcal{O}(N^2 + T(K + N + \delta_v^2 C) + L N \log N)$, where $N$ and $K$ are the total numbers of nodes and hyperedges, $T$ is the number of distinct thresholds in a dendrogram. $\delta_v$ is the maximum node degree, and $C$ is the average number of clusters in the dendrogram, with $C \ll N$. 
\end{proposition}

\begin{proof}
        The time complexity of constructing $L$ PHPool operators can be broken into three steps:
    \begin{enumerate}
        \item The construction of the hierarchical clustering dendrogram takes $O(N^2)$ using the nearest-neighbor chain algorithm~\cite{bar2001fast}.
        \item To select the optimal thresholds from the dendrogram, we need to examine $T$ thresholds, and computing the modularity at a particular threshold takes $\mathcal{O}(K + N + \delta_v^2 C)$, where $\delta_v$ denotes the maximum degree of nodes and $C$ is the average number of clusters. Overall, the modularity-based selection takes $\mathcal{O}(T(K + N + \delta_v^2 C))$ time.
        \item To determine consecutive layer-to-layer mappings, for every two connected layers, we examine each cluster by finding common parents of the nodes in two layers, the step of searching for common-clustered parents requires $\mathcal{O}(N \log N)$, so overall $\mathcal{O}(L N \log N)$.
        \end{enumerate}

      Hence, we obtain the total complexity: $\mathcal{O}(N^2 + T(K + N + 
       \delta_v^2 C) + L N \log N)$.
\end{proof}

\noindent {\bf Remark.} The most time-consuming part of constructing $L$ PHPool layers is the construction of the dendrogram that takes $O(N^2)$. The three terms map to the PHPool stages in Section~\ref{subsec:phpool}: $O(N^2)$ for shortest-path distances and agglomerative dendrogram construction; $O(T(K + N + \delta_v^2 C))$ for modularity-guided threshold selection; and $O(LN\log N)$ for the parent--child mapping underlying clique-based hyperedge reconstruction (and its PHUnpool counterpart). This pipeline runs once as preprocessing, so it does not affect per-epoch training speed; the runtime in Table~\ref{table: runtime} is measured \emph{after} PHPool assignments are precomputed. The $O(N^2)$ bottleneck can be further mitigated by approximate hierarchical clustering ($O(N\log N)$)~\cite{moseley2023approximation, charikar2017approximate, manghiuc2021hierarchical} or distributed schemes ($O(|\mathcal{E}|N)$)~\cite{bateni2017affinity}.

\begin{proposition}
    The time complexity of the proposed HGXConv for each layer is $\mathcal{O}(K\delta_e h + Kh^2 + N \delta_v h + N h^2)$, where $N$ and $K$ are the number of nodes and hyperedges, respectively. $\delta_v$ and $\delta_e$ are the maximum node degree and hyperedge degree. $h$ denotes the feature dimension.
\end{proposition}

\begin{proof}
     The first two terms in the expression correspond to the complexity of node-to-hyperedge propagation, where $K\delta_e h$ is the complexity for aggregating node features, $Kh^2$ is complexity for transforming node features using MLP. Similarly, the last two terms represent the complexity for hyperedge-to-node propagation. In addition, the fusion step takes $\mathcal{O}(2Nh^2)$, so the overall time complexity for each HGXConv layer is $\mathcal{O}(K\delta_e h + Kh^2 + N \delta_v h + N h^2)$.
\end{proof}

\subsection{Permutation Equivariance}
In hypergraph representation learning, an important property is to ensure permutation equivariant or invariant. We show that PHPool is permutation equivariant. Permutation equivariant means that if the nodes of the graph are permuted, the corresponding output of the model should reflect that permutation in a consistent way.

\begin{definition}{\bf \textit{Permutation Equivariance for Hypergraphs}.} 
Given a permutation function $\pi: \mathcal{V} \to \mathcal{V}$ such that, for any node $v_i, i \in \{1, 2, ..., N\}$, there exists exactly one index $j$ such that $\pi(j)=i$. Let $\mathcal{G}$ be any  hypergraph with node feature matrix $\mathbf{F}$, a function $f: (\mathcal{G}, \mathbf{F}) \to y$ is permutation equivariant \textit{iff} $f(\pi(\mathcal{G}), \pi(\mathbf{F}))=\pi(f(\mathcal{G}, \mathbf{F})) $. 
\end{definition}
\begin{proposition}
The HGUN is permutation equivariant. 
\end{proposition}

\begin{proof}
     For PHPool layers, the cluster assignment matrix $\mathbf{S}$ generation is permutation equivariant because the hierarchical dendrogram will permute consistently if node order changes. Furthermore, from Eq.~\eqref{eq:PHPool}, we can see that hypergraph permutation does not alter the pooled hypergraph features; it only changes the order in which the features are computed. Regarding HGXConv layers, the two-stage propagations $f_{\mathcal{V}\to \mathcal{E}}$ in Eq.~\eqref{eq:cross_node} and $f_{\mathcal{E}\to \mathcal{V}}$ in Eq.~\eqref{eq:edge_to_node} are permutation invariant because they are functions operated on an incident node sets and hyperedge sets. Since the fusion step in Eq.~\eqref{eq:fusion} requires node features as input, it is permutation equivariant, which makes the overall HGXConv layer permutation equivariant. Overall, the HGUN is permutation equivariant.
\end{proof}

\subsection{Expressive power of hypergraph pooling operators}~\label{subsec:expressive}
\par
The expressive power of pooling operations is defined as the capability of retaining the same expressive power as the message-passing layers before it~\cite{bianchi2024expressive}. We show the PHPool layer satisfies three sufficient (not necessary) conditions for examining the expressive power of hypergraph pooling operators. The theorem is as follows

\begin{theorem}\label{theorem:pool_express}
    Let $ \mathcal{G}_1 =(\mathcal{V}_1, \mathcal{E}_1) $ with $|\mathcal{V}_1| = N_1$ and $ \mathcal{G}_2 =(\mathcal{V}_2,\mathcal{E}_2) $ with $|\mathcal{V}_2| = N_2$ be two nonisomorphic hypergraphs. Suppose after applying a block of several message-passing layers, the two hypergraphs are $\mathcal{G}_{1}^{mp}$ and $\mathcal{G}_{2}^{mp}$ with node features $\mathbf{F}_{1}^{mp} \in \mathbb{R}^{N_1\times D}$ and $\mathbf{F}_{2}^{mp} \in \mathbb{R}^{N_2\times D}$. For a given pooling operator $\mathtt{POOL}$, let $\mathcal{G}_1^{pool} = \mathtt{POOL}(\mathcal{G}_1^{mp})$ and $\mathcal{G}_2^{pool} = \mathtt{POOL}(\mathcal{G}_2^{mp})$ have the same number of nodes $|\mathcal{V}_1^{pool}| = |\mathcal{V}_2^{pool}| = N$. The proposed HGXConv layer and the PHPool layer satisfy the following three conditions:

    \begin{enumerate}
        \item $ \sum_i^{N_1} \mathbf{f}_i^{mp} \neq \sum_j^{N_2} \mathbf{f}_j^{mp}$
        \item The cluster assignment matrix $\mathbf{S}$ by $\mathtt{POOL}$ satisfies $\sum_{j=1}^{N} s_{ij} = \lambda$, and $\lambda>0,  \forall i$. 
        \item The coarsened node feature $\mathbf{F}^{pool}$ is a cluster-wise aggregation of the unpooled node feature $\mathbf{F}^{mp}$,
    \end{enumerate}
 \hspace{.15in} such that the pooling operator $\mathtt{POOL}$ will produce different node features $\mathbf{F}_1^{pool}\in \mathbb{R}^{N\times D}$ and $\mathbf{F}_2^{pool}\in \mathbb{R}^{N\times D}$ for any row permutations.
\end{theorem}
Condition 1 is a constraint on the message-passing layers before a pooling layer; it requires that the message-passing layers have discriminative power to obtain different sums of node embeddings for two non-isomorphic hypergraphs. Condition 2 requires all nodes in the original hypergraph must contribute to the construction of new supernodes, and the contributions for all nodes are on the same scale. Condition 3 further requires the supernode features to be a cluster-wise aggregation (i.e., a linear combination determined by the cluster assignment matrix) of original node features.

\begin{proof}
    For Condition 1, the hypergraph isomorphism computation~\cite{hyperg_isomorphism} shows that the key requirement for message passing to be expressive is to ensure message passing is injective.  The node-to-hyperedge aggregation function $f_{\mathcal{V}\to \mathcal{E}}$, which computes the element-wise product of the node features, is unfortunately not injective. The hyperedge-to-node summation function $f_{\mathcal{E}\to \mathcal{V}}$, on the other hand, is shown to be injective. The fusion step in Eq.~\eqref{eq:fusion} preserves the self-feature alongside the neighborhood embedding, which mitigates the non-injectivity of the multiplicative node-to-hyperedge aggregation in Eq.~\eqref{eq:cross_node} and improves discriminative capacity in practice.

PHPool assigns each node in the original hypergraph to exactly one cluster, satisfying condition 2: $\sum_{j=1}^N s_{ij} = 1, \forall i$. Furthermore, since our feature coarsening is expressed as $\mathbf{F}^{pool} = \mathbf{S}^T\mathbf{F}^{mp}$, $\mathbf{F}^{pool}$ is a cluster-wise aggregation (a sum of features over members of each cluster) of the unpooled node feature $\mathbf{F}^{mp}$, satisfying Condition~3.
\end{proof}

\subsection{Information Loss Comparison Between Sequential and Parallel Pooling}
\label{subsec:mutual_info}

Here we present the proposition that shows parallel pooling has less information loss than sequential pooling under matched first-level coarsening. This result formalizes one mechanism by which parallel pooling can reduce cumulative information loss relative to sequential pooling, but it should not be read as an unconditional guarantee for arbitrary parallel and sequential pooling maps.

{

\begin{proposition}\notag\label{prop:info_loss_compare}
Let $A$ denote the original hypergraph. Assume the first-level pooled representation is matched in the sequential and parallel constructions, i.e., $B_{\mathrm{seq}}=B_{\mathrm{par}}=B$, or more generally that $I(A;B_{\mathrm{par}})=I(A;B_{\mathrm{seq}})$. Consider
(i) \emph{sequential pooling}, which produces $C_{\mathrm{seq}}$ from $B$ so that $A \to B \to C_{\mathrm{seq}}$, and
(ii) \emph{parallel pooling}, which produces $C_{\mathrm{par}}$ directly from $A$ (with internal randomness independent of $B$ conditional on $A$).

Define the information loss in the pooling procedure as $\operatorname{InfoLoss}(A, B, C) : = H(A) - I(A; B, C)$, where $H(A)$ is the entropy of the original hypergraph $A$, $I(A; B, C)$ is the mutual information between $A$ and the two pooled hypergraph $B, C$. Under the matched first-level coarsening assumption above, the information loss in parallel pooling is at most that in sequential pooling, i.e.,  $\operatorname{InfoLoss}(A, B, C_{\mathrm{par}})  \leq \operatorname{InfoLoss}(A, B, C_{\mathrm{seq}})$.

\end{proposition}
}

{%

\begin{proof}
Denote $B_{\mathrm{seq}}=B_{\mathrm{par}}=B$. The pair \((A,B_{\mathrm{seq}},C_{\mathrm{seq}})\) forms a Markov chain
\(A \!\to\! B_{\mathrm{seq}}\!\to\! C_{\mathrm{seq}}\), hence
\(I(A;C_{\mathrm{seq}}\mid B_{\mathrm{seq}})=0\).
Applying the chain rule \(I(A;B,C)=I(A;B)+I(A;C\mid B)\) gives
\begin{equation}\label{eq:seq}
  I_{\mathrm{seq}}
  := I(A;B_{\mathrm{seq}},C_{\mathrm{seq}})
  = I(A;B_{\mathrm{seq}}).
\end{equation}

\noindent For the parallel architecture no conditional independence holds, so
\begin{equation}\label{eq:par}
  I_{\mathrm{par}}
  := I(A;B_{\mathrm{par}},C_{\mathrm{par}})
  = I(A;B_{\mathrm{par}})
    + I\bigl(A;C_{\mathrm{par}}\mid B_{\mathrm{par}}\bigr).
\end{equation}

\noindent Because the first pooling map is determined by the original hypergraph $A$ in both parallel and sequential pooling, we have
\(I(A;B_{\mathrm{par}})=I(A;B_{\mathrm{seq}})\).
Subtracting \eqref{eq:seq} from \eqref{eq:par} yields
\[
  I_{\mathrm{par}} - I_{\mathrm{seq}}
  = I\bigl(A;C_{\mathrm{par}}\mid B_{\mathrm{par}}\bigr) \;\ge\; 0.
\]

\noindent Finally, we obtain $\operatorname{InfoLoss}_{\mathrm{par}}
  = H(A)-I_{\mathrm{par}}
  \le H(A)-I_{\mathrm{seq}}
  = \operatorname{InfoLoss}_{\mathrm{seq}}.$
\end{proof}
}%

\section{Experiments}
In this section, we conduct comprehensive experiments to assess the performance of our proposed model, which includes hypergraph reconstruction, hypergraph classification, and node classification. The experimental results are presented and discussed. 

\subsection{Hypergraph Reconstruction on Synthetic Networks}\label{subsec:synthetic}
Recall that one of the most important goals in node pooling methods is to preserve as much network structure as possible. As a result, before presenting the experiments in real-world datasets, we first design an experiment on synthetic networks to examine the compression-and-recovery capability of HGUN.

{\bf Data generation.} We generate four hypernetworks: ring, grid, pyramid, and community, and let the coordinate of each node be the node feature matrix $\mathbf{F}\in \mathbb{R}^{N\times 2}$. The ring network contains $200$ nodes and $200$ hyperedges, where each hyperedge is constructed by a node and the left and right neighbors, so the network is a 3-uniform hypergraph. The grid network is constructed by $30 \times 30$ nodes. Similar to the ring network, each hyperedge in the grid network is a small square that contains a node and the right, bottom, and bottom-right neighbors. The pyramid network has $210$ nodes and each hyperedge is a 3-point triangle. These three synthetic hypergraphs all have \textit{uniform} structures, in which cluster behavior is rare. So we further generate a community network to simulate more realistic network structures.

{\bf Experiment setup.} 
The learning objective is to minimize the mean squared error (MSE) between the input coordinates and recovered coordinates. To perform a fair comparison, we use identical u-net architectures and hyperparameters. Following prior work~\cite{bianchi2020spectral}, each u-net has one block of HGXConv + PHPool/ PH-UnPool, and an additional convolution layer at the end. The reason for having an additional convolutional layer is to use it as the final prediction projector. We selected two representative baselines: TopKPool~\cite{graphunet} from node dropping and MinCutPool~\cite{bianchi2020spectral} from node clustering. Regarding hyperparameters, we set the pooling ratio of all models to $25\%$, the learning rate to $0.005$, and the hidden dimension to $32$. We train each model for up to $10,000$ epochs and conduct early stopping with patience of $1000$ epochs. For the two graph autoencoders used for TopKPool and MinCutPool, we use GCN~\cite{GCN} as convolution layers. To control factors other than the pooling approach, we use HGNN~\cite{HGNN} as convolution layers for HGUN rather than the proposed HGXConv. This is because HGNN performs a clique expansion to hypergraphs and adapts GCN to the reduced hypergraph, which is computing the same message passing as GCN.

\begin{figure}[ht]
  \centering
  \includegraphics[width=\linewidth]{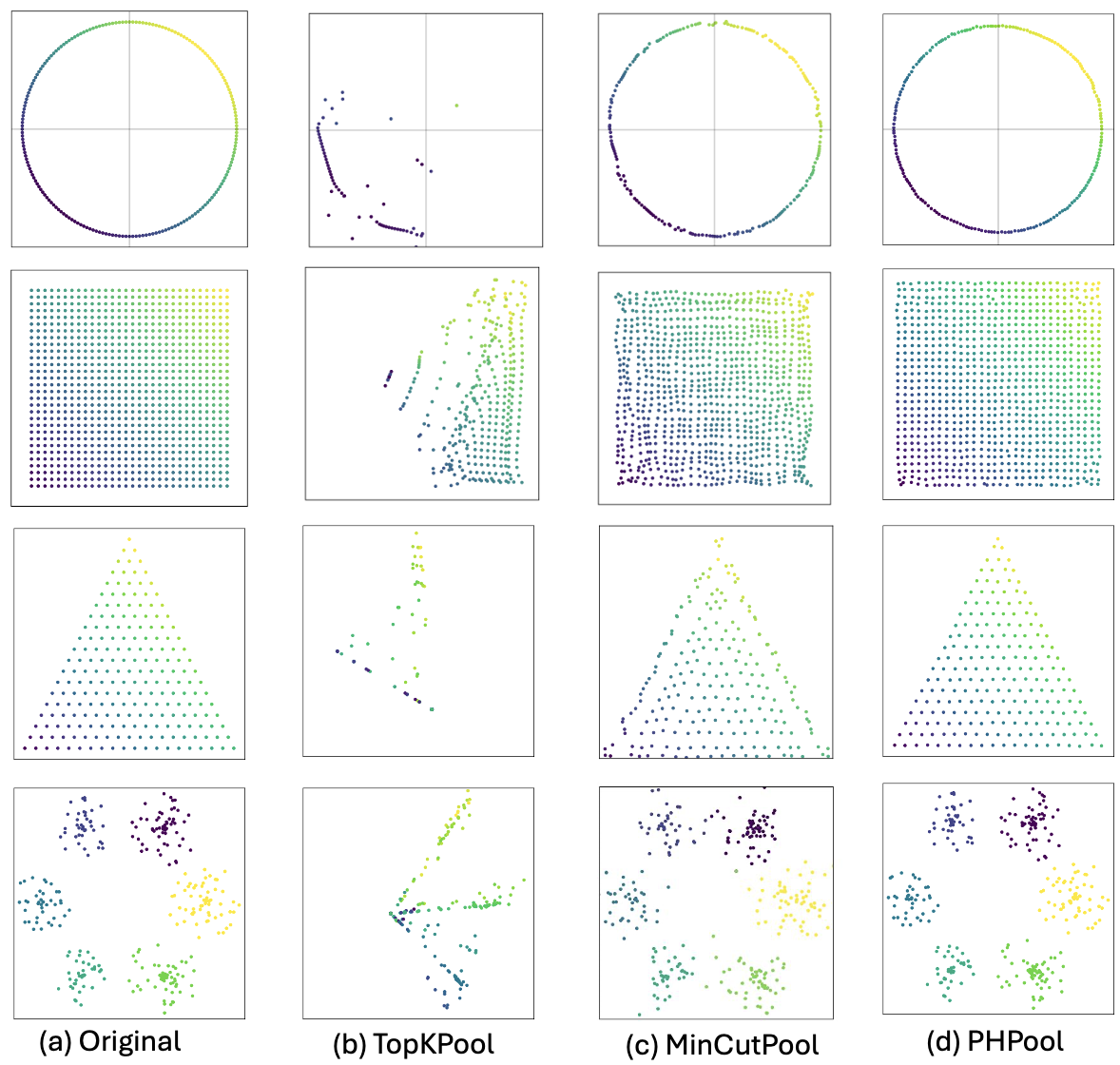}
  \vspace{-0.7cm}
  \caption{Signal reconstructions on ring, grid, pyramid, and community networks with different pooling methods.}
  \label{fig:sig_recon}
\end{figure}

{\bf Reconstruction results and discussion} Fig.~\ref{fig:sig_recon} presents the original network signals $\mathbf{F}$ and the reconstructed signals $\mathbf{F}^{rec}$ using graph or hypergraph autoencoders. On all of the networks, PHPool consistently outperforms the other two benchmarks. We first observe the reconstruction results for TopKPool suffers from the most severe structural damage on all of these networks. On the ring, grid, and pyramid datasets, partial structures are retained, but for community networks that are common structures in real-world dataset, TopKPool generates random coordinates that make no sense. This is not surprising because discarding nodes could cause sparse neighboring aggregation and harm the eventual coordinate prediction. With respect to MinCutPool, we can see that it has improved upon TopKPool in the sense that it retains the basic shape of the original topology. However, we can see that on the boundaries of the grid and the pyramid network, MinCutPool still suffers from local structural distortion. What's more, as a node-cluster pooing method, MinCutPool widely spreads the node locations, adding more noise to the node coordinates especially those between clusters in the community network. Imaging if k-means clustering is applied based on the node coordinates, then there will be more nodes that lie on the border that are hard to cluster. This validates the assumption that sequentially generating pooling operators lead to a larger information loss during data processing. The superior performance of PHPool demonstrates the effectiveness of our parallel-pooling design, suggesting the original network structure is well preserved in the embedding space. 

\begin{table*}[ht]
\begin{center}
\caption{Average testing accuracy ($\%$, $\pm$ standard deviation) on 9 real-world datasets for hypergraph classification. The best results are highlighted for each dataset.}
\label{tab:graph_class_result}
\resizebox{\textwidth}{!}{%
\begin{tabular}{ccccccccc}
\hline
& & \multicolumn{3}{c}{Social Network} & \multicolumn{4}{c}{Bioinformatics} \\
\hline
& & IMDB-BINARY & IMDB-MULTI & COLLAB & MUTAG & PROTEINS & D\&D & NCI1 \\
\# Graphs & & 1,000 & 1,500 & 5,000 & 188 & 1,113 & 1,178 & 4,110 \\
\# Classes & & 2 & 3 & 3 & 2 & 2 & 2 & 2 \\
Avg. \# Nodes & & 19.8 & 13.0 & 74.5 & 17.9 & 39.1 & 284.3 & 29.8 \\
\hline

\multirow{5}{*}{\begin{tabular}{c}Conv. \\Networks\end{tabular}}
& $\mathrm{GCN}$ & $73.26 \pm 0.46$ & $50.39 \pm 0.41$ & $80.59 \pm 0.27$ & $69.50 \pm 1.78$ & $73.24 \pm 0.73$ & $72.05 \pm 0.55$ & $76.29 \pm 1.79$ \\
& HGNN & $72.78 \pm 0.86$ & $48.13 \pm 1.36$ & $78.19 \pm 0.63$ & $71.29 \pm 1.38$ & $71.46 \pm 1.66$ & $73.79 \pm 1.17$ & $75.00 \pm 1.40$ \\
& HGAT & $74.02 \pm 0.32$ & $50.87 \pm 0.62$ & $81.44 \pm 0.37$ & $74.53 \pm 0.98$ & $71.52 \pm 0.74$ & $73.98 \pm 0.57$ & $76.58 \pm 0.85$ \\
& HNHN & $67.43 \pm 0.79$ & $50.93 \pm 0.83$ & $79.67 \pm 0.42$ & $73.88 \pm 1.74$ & $72.13 \pm 1.24$ & $73.58 \pm 1.26$ & $75.27 \pm 1.05$ \\
& HGXConv & $73.34 \pm 0.65$ & $50.76 \pm 0.83$ & $81.24 \pm 0.43$ & $73.37 \pm 1.09$ & $73.92 \pm 0.98$ & $73.67 \pm 0.88$ & $76.05 \pm 0.73$ \\
\hline

\multirow{6}{*}{\begin{tabular}{c}Graph \\Pooling \\  w. Default \\Backbones\end{tabular}}
& ASAP+GCN & $72.81 \pm 0.50$ & $50.78 \pm 0.75$ & $78.64 \pm 0.50$ & $77.83 \pm 1.49$ & $73.92 \pm 0.63$ & $76.58 \pm 1.04$ & $71.48 \pm 0.42$ \\
& SAGPool+GCN & $72.55 \pm 1.28$ & $50.23 \pm 0.44$ & $78.03 \pm 0.31$ & $73.67 \pm 4.28$ & $71.56 \pm 1.49$ & $74.72 \pm 0.82$ & $67.45 \pm 1.11$ \\
& TopKPool+GCN & $71.58 \pm 0.95$ & $48.59 \pm 0.72$ & $77.58 \pm 0.85$ & $67.61 \pm 3.36$ & $70.48 \pm 1.01$ & $73.63 \pm 0.55$ & $67.02 \pm 2.25$ \\
& DiffPool+GraphSage & $73.14 \pm 0.70$ & $51.31 \pm 0.72$ & $78.68 \pm 0.43$ & $79.22 \pm 1.02$ & $73.03 \pm 1.00$ & $77.56 \pm 0.41$ & $62.32 \pm 1.90$ \\
& MinCutPool+GCN & $72.65 \pm 0.75$ & $51.04 \pm 0.70$ & $80.87 \pm 0.34$ & $79.17 \pm 1.64$ & $74.72 \pm 0.48$ & $78.22 \pm 0.67$ & $74.25 \pm 0.86$ \\
& SEPool+GCN & $74.12 \pm 0.56$ & $51.53 \pm 0.65$ & $81.28 \pm 0.15$ & $85.56 \pm 1.09$ & $76.42 \pm 0.39$ & $77.98 \pm 0.57$ & $78.35 \pm 0.33$ \\
\hline

\multirow{3}{*}{\begin{tabular}{c}Graph \\pooling w.\\HGXConv\end{tabular}}
&ASAP+HGXConv & $73.66 \pm 0.43$ & $51.27 \pm 0.63$ & $79.31 \pm 0.48$ & $79.83 \pm 1.77$ & $74.55 \pm 0.77$ & $76.93 \pm 0.62$ & $72.81 \pm 1.03$ \\
&TopKPool+HGXConv & $72.31 \pm 0.44$ & $50.83 \pm 0.98$ & $68.55 \pm 0.50$ & $69.88 \pm 1.98$ & $71.56 \pm 0.83$ & $75.66 \pm 0.73$ & $68.23 \pm 1.11$ \\
&SEPool+HGXConv & $74.31 \pm 0.53$ & $51.73 \pm 0.71$ & $80.73 \pm 0.44$ & $85.02 \pm 1.33$ & $77.31 \pm 0.44$ & $76.93 \pm 0.56$ & $78.88 \pm 0.98$ \\
\hline

\multirow{3}{*}{\begin{tabular}{c}PHPool w.\\Conv. \\Networks\end{tabular}}
&PHPool+HGNN & $74.88 \pm 0.59$ & $53.21 \pm 0.77$ & $81.98 \pm 0.49$ & $85.33 \pm 1.54$ & $73.88 \pm 0.59$ & $77.93 \pm 0.64$ & $79.93 \pm 1.41$ \\
&PHPool+HGAT & $75.36 \pm 0.36$ & $53.78 \pm 0.69$ & $82.33 \pm 0.44$ & $\mathbf{86.73 \pm 1.39}$ & $\mathbf{78.33 \pm 0.48}$ & $77.43 \pm 0.58$ & $81.23 \pm 0.94$ \\
&PHPool+HNHN & $72.67 \pm 0.53$ & $52.81 \pm 0.73$ & $81.04 \pm 0.52$ & $84.83 \pm 1.60$ & $74.32 \pm 0.54$ & $74.58 \pm 0.55$ & $80.02 \pm 0.96$ \\
\hline

\multirow{1}{*}{\begin{tabular}{c}HGUN(Enc)\end{tabular}}

& PHPool+HGXConv & $\mathbf{76.53 \pm 0.48}$ & $\mathbf{55.31 \pm 0.61}$ & $\mathbf{83.26 \pm 0.53}$ & $86.42 \pm 1.65$ & $78.06 \pm 0.56$ & $\mathbf{79.88 \pm 0.54}$ & $\mathbf{81.89 \pm 1.03}$ \\
\hline

\end{tabular}
}
\end{center}
\end{table*}

\subsection{Hypergraph Classification}\label{subsec:hypergraph_classification}
Hypergraph classification aims to predict the label associated with an entire hypergraph. Applications of hypergraph classification include social network classification (e.g., type of an ego-network), molecule classification (e.g., mutagenetic or not), etc. Datasets for hypergraph classification usually contain many small-sized hypergraphs and can be easily trained in batches.

{\bf Dataset.} We select seven datasets from TU datasets~\cite{morris2020tudataset} for hypergraph classification, they include: 

\noindent {\bf Social network dataset.}
IMDB-BINARY and IMDB-MULTI are datasets that revolve around movie collaborations. Each dataset represents movies as graphs where vertices correspond to actors/actresses. An edge is established between two actors/actresses if they appear together in another film. IMDB-BINARY focuses on movies from two genres—Action and Romance—while IMDB-MULTI includes an additional genre, Sci-Fi, alongside Comedy and Romance, making it a multiclass dataset. Similarly, the COLLAB dataset is a collaboration dataset in scientific domains. Each hypergraph in this dataset represents the ego network of a researcher, with edges denoting collaboration between two researchers. Each researcher's ego network is classified into one of three possible fields: High Energy Physics, Condensed Matter Physics, or Astro Physics, reflecting the specific area of physics to which the researcher is affiliated.

\noindent {\bf Biochemical dataset.} The MUTAG dataset includes datasets of mutagenic aromatic and heteroaromatic nitro compounds and the goal is to predict their mutagenicity. The PROTEINS dataset, on the other hand, contains nodes that represent secondary structure elements (SSEs) of proteins. Connections between nodes signify neighboring elements in the amino acid sequence or three-dimensional space. The task of the PROTEINS dataset is to classify a network as one of the three types: helix, sheet, and turn. The D\&D dataset consists of graphs that depict protein structures, where each node represents an amino acid. Edges between nodes are established if the distance between two amino acids is less than 6 angstroms. Each protein in the dataset is labeled as either an enzyme or a non-enzyme, indicating its functional category. NCI1 is a dataset collected by the National Cancer Institute (NCI) and contains chemical compounds for anti-cancer screens where the task is to classify chemicals as positive or negative to cell lung cancer.

\noindent {\bf Conversion of graphs to hypergraphs.} The initial datasets are structured as graphs, representing network structures. While there are studies that learn hypergraph topology from node features~\cite{pena2023learning, tang2023learning}, node features are not always available in datasets. As a result, we choose to simply elevate graph structures to hypergraphs by identifying higher-order relationships within them. In the three social network datasets, hyperedges are formulated for groups of nodes that are mutually connected (i.e., forming cliques). For the four biochemical datasets, the significance of ring structures in molecules for identifying their properties has been highlighted in previous studies~\cite{debnath1991structure, konstantinova1998graph} Consequently, we construct hyperedges based on these ring structures in the biochemical datasets.

{\bf Baselines.}  Without pooling layers, a hypergraph convolutional neural network can only learn node embeddings, and a common practice to make hypergraph-level predictions is to summarize the node embeddings through a sum or average readout function at the end of the neural network. We include four convolutional hypergraph neural networks: HGNN~\cite{HGNN}, HNHN~\cite{hnhn}, HGAT.~\cite{bai2021hypergraph} and the proposed HGXConv as our baselines. Since these datasets are initially constructed as simple graphs, we also include a simple graph neural network (GCN~\cite{GCN}). We refer to these (hyper)graph convolutional neural networks as backbone models because they are widely used in autoencoder architectures. To validate PHPool, we then consider three node-drop pooling methods as well as three node-cluster pooling methods.

\noindent {\bf Node-drop pooling methods.}
\begin{itemize}[leftmargin=*]
\item TopKPool~\cite{graphunet} selects nodes according to an importance score learned by the product of the node features and a trainable projection vector. 
    \item   SAGPool~\cite{lee2019self} replaces the projection vector with an output of a self-attention message-passing layer, aiming to take into account the graph structure during score learning.
    \item  ASAP~\cite{ranjan2020asap} considers a fixed radius of each node as a cluster and learns a score function to sample a subset of each node neighborhood through local convolution. Though it makes use of the cluster assignment idea when learning the score function, it still falls into the node-drop pooling category because it discards nodes with lower scores.
\end{itemize}

\noindent {\bf Node-cluster pooling methods.}
\begin{itemize}[leftmargin=*]
    \item DiffPool~\cite{ying2018hierarchical} employs graph neural networks to learn a cluster assignment matrix, supplemented by an auxiliary link prediction loss to encourage the formation of clusters on neighboring nodes.
    \item MinCutPool~\cite{bianchi2020spectral} relaxes the learning objective of spectral clustering and includes it in the loss function. The cluster assignment is learned by multi-layer perceptrons.
    \item SEPool~\cite{SEP} constructs coding trees for graphs by minimizing structural entropy at specific depth levels, any two consecutive layers of the coding tree define a cluster assignment mapping.
\end{itemize}

{\bf Experiment setup.} We follow previous work~\cite{bianchi2020spectral, SEP} and evaluate our model performance with a 10-fold cross-validation. Based on the training/testing split provided by~\cite{errica2019fair}, we hold $10\%$ of the training data as a validation set to select the best model hyperparameters. Specifically, the number of depth for all models are set to $3$, and the pooling ratio is $25\%$ for the models that require a pooling ratio.  The hidden dimension is tuned among $\{32, 64, 128, 256\}$, the batch size is selected from $\{16, 64, 128\}$, and the dropout rate is $0$ or $0.5$. The neural network is optimized with an Adam optimizer, with a learning rate $5\times 10^{-4}$ and weight decay $1\times 10^{-4}$. An early stopping mechanism is used to avoid overfitting, if the model performance does not approve on the validation set for $100$ consecutive epochs, we terminate the training procedure. For each fold of the dataset, we repeat the experiment $10$ times with different seeds, so overall a total of $100$ experiments are run on a dataset.

\noindent \textbf{Results and discussion.}

Table~\ref{tab:graph_class_result} summarizes the hypergraph classification results on social network and bioinformatics datasets. Overall, HGUN(Enc) (i.e., PHPool+HGXConv) exhibits consistently strong performance across benchmarks and ranks first on 5 out of 7 datasets. On the two bioinformatics datasets MUTAG and PROTEINS, the top-performing model is PHPool+HGAT (86.73 on MUTAG and 78.33 on PROTEINS), while HGUN(Enc) remains a close second (85.42 on MUTAG and 78.06 on PROTEINS). This gap suggests that these two molecular datasets may benefit from more expressive and fine-grained interaction weighting, where attention mechanisms can better adaptively emphasize informative node and hyperedge relations (e.g., different bond types or functional groups). Importantly, even without attention, PHPool+HGXConv remains comparable to PHPool+HGAT, indicating that (i) PHPool effectively captures global hierarchical representations at multiple scales, and (ii) the multiplicative higher-order mixing in HGXConv can still provide competitive expressiveness when coupled with hierarchical pooling.

A key takeaway from Table~\ref{tab:graph_class_result} is that PHPool consistently unlocks global, multi-resolution structure information and improves performance across different hypergraph convolution backbones. Comparing each convolution network against its PHPool counterpart shows that PHPool provides substantial gains, particularly on bioinformatics datasets. For example, HGAT alone achieves 74.53 on MUTAG, whereas PHPool+HGAT improves it to 86.73, demonstrating the effectiveness of PHPool in extracting task-relevant hierarchical structure beyond what a flat convolution can capture. Similar improvements are observed for other convolution backbones, reinforcing that the benefits of PHPool are not tied to a specific convolution design but rather stem from modularity-guided parallel hierarchical coarsening.

We further validate the advantage of PHPool by controlling the convolution backbone in pooling baselines. When equipping other SOTA pooling methods with HGXConv (e.g., ASAP/TopK/SEPool+HGXConv), PHPool+HGXConv achieves the best performance among these HGXConv-based pooling variants, supporting the superiority of PHPool as a hypergraph pooling operator. Finally, while HGXConv itself is not always the strongest standalone hypergraph convolution, combining HGXConv with PHPool yields the best overall performance on most benchmarks. This supports our positioning that PHPool contributes the most significant gains in the HGUN architecture, while HGXConv provides an additional, complementary inductive bias for modeling higher-order interactions through multiplicative mixing.

\subsection{Depth Sensitivity for Hypergraph Classification}\label{subsec: hypergraph_classification_depth}
We examine pooling depth effects across COLLAB and D\&D. We choose COLLAB because it contains the largest number of graphs and D\&D because it has the largest average number of nodes. Fig.~\ref{fig:graph_classification_depth_effect} shows optimal performance at depth 3 for all methods. Beyond this depth, COLLAB exhibits sharp degradation at depth 5, likely from over-compression of dense community structures, while D\&D degrades more gradually, suggesting hierarchical abstraction benefits larger, sparser graphs. Critically, TopK's node-drop strategy collapses on D\&D at depth 4, dropping 11 points. We suspect this could be due to cascading selection errors where irreversible node dropping compounds information loss. In contrast, node-clustering methods (HGUN, MinCut) maintain stability through soft aggregation.
\begin{figure}[ht]
  \centering
  \includegraphics[width=\linewidth]{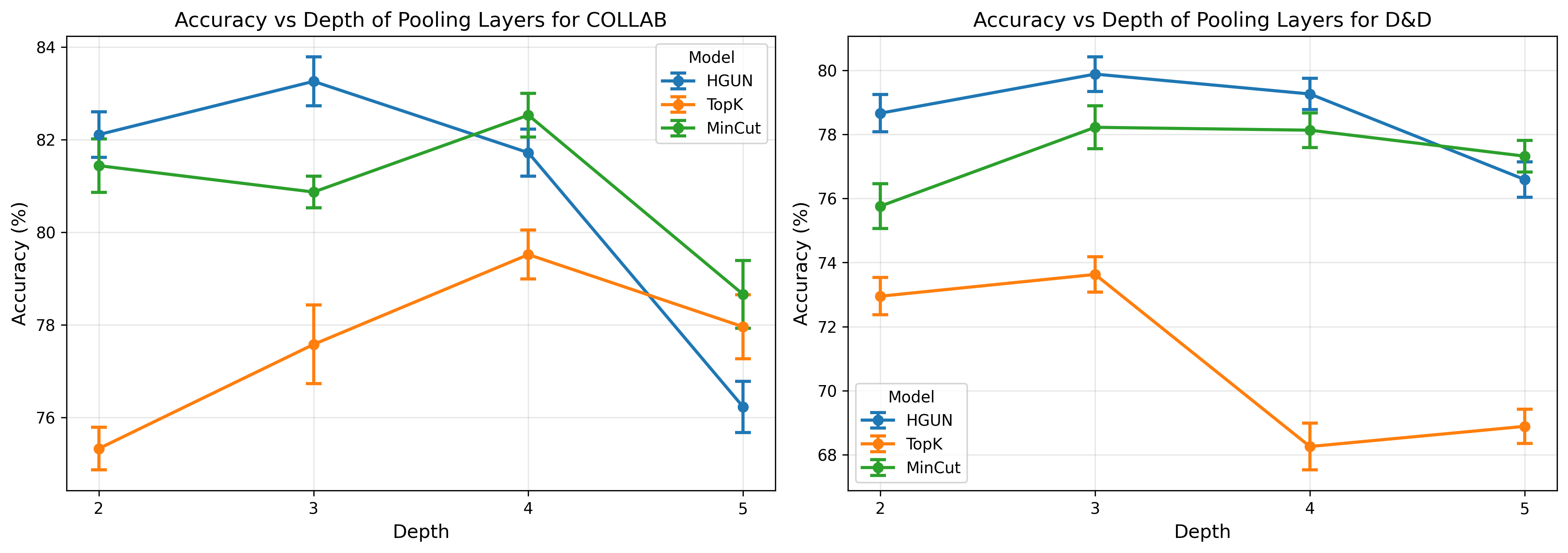}
  \vspace{-0.7cm}
\caption{Classification accuracy w.r.t depths on COLLAB and D\&D datasets for hypergraph classification. Comparison between HGUN encoder and 2 graph pooling methods.}
  \label{fig:graph_classification_depth_effect}
\end{figure}

\begin{table*}[h!]
    \centering
    \caption{Experiment results for node anomaly detection. The best average metrics are highlighted for each dataset.}
\begin{tabular}{c|l|cccc|cccc}
\hline & Dataset & \multicolumn{4}{c}{ Yelp } & \multicolumn{4}{c}{ Amazon } \\
\hline 
 Model & Metric & Precision & GMean & F1-macro & Recall & Precision & GMean & F1-macro & Recall \\
\hline 


\multirow{4}{*}{\begin{tabular}{c} 
GNNs
\end{tabular}} 
& GCN & 0.3155 & 0.4946 & 0.3667 & 0.7753 & 0.6793 & 0.7372 & 0.5643 & 0.8000 \\
 &GAT & 0.4542 & 0.5313 & 0.4277 & 0.6215 & 0.8619 & 0.8304 & 0.7146 & 0.8000 \\
& GPRGNN & 0.6489 & 0.6984 & 0.5734 & 0.7516 & 0.8461 & 0.8232 & 0.6415 & 0.8009 \\
& FAGCN & 0.7108 & 0.7086 & 0.6111 & 0.7064 & 0.8757 & 0.8433 & $0.7330$ & 0.8121 \\
\hline 

\multirow{4}{*}{\begin{tabular}{c} 
AD \\
GNNs
\end{tabular}} 
& CARE-GNN & 0.6943 & 0.7086 & 0.6040 & 0.7232 & $0.8813$ & 0.8171 & 0.7245 & 0.7576 \\
& RioGNN & 0.6882 & 0.7352 & 0.6161 & 0.7854 & 0.8609 & 0.8204 & 0.7262 & 0.7818 \\
& H²-FDetector & 0.7877 & 0.8164 & 0.7078 & $\mathbf{0.8461}$ & 0.8652 & 0.8429 & 0.7236 & $\mathbf{0.8394}$  \\
& BWGNN & $0.8128$ & $0.8192$ & $0.7232$ & $0.8256$ & $0.8541$ & $0.8467$ & 0.7163 & $0.8212$ \\
\hline
\multirow{2}{*}{\begin{tabular}{c} 
Graph U-Nets\\
\end{tabular}} 
& TopK + GCN& 0.7033 & 0.7335 & 0.7133 & 0.6779  & 0.8226 & 0.8180 & 0.7023 & 0.7697 \\
 &MinCut + GCN& 0.8372 & 0.7623 & 0.7281
 & 0.6941 & 0.8136 & 0.8276 & 0.7244 & 0.7606 \\
\hline 
\multirow{4}{*}{\begin{tabular}{c} 
HyperGNNs\\
\end{tabular}} 
& HGNN & 0.4173 & 0.5650 & 0.3421 & 0.7650 & 0.7531 & 0.7704 & 0.6025 & 0.7880 \\
& HNHN& 0.6573 & 0.6395 & 0.6528 & 0.6221 & 0.8024 & 0.8101 & 0.7129 & 0.8179 \\
& HGAT & 0.4093 & 0.4751 & 0.4698 & 0.5515 & 0.7459 & 0.7688 & 0.6933 & 0.7922 \\
 &HGXConv& 0.6863 & 0.7237 & 0.7227 & 0.7632 & 0.8372 & 0.8291 & 0.7090 & 0.8211 \\
 \hline
 \multirow{1}{*}{\begin{tabular}{c} 
HyperGraph U-Net \\
\end{tabular}} 
&HGUN & $\mathbf{0.8523}$ & $\mathbf{0.8272}$ & $\mathbf{0.7732}$ & 0.8026 & $\mathbf{0.8867}$ & $\mathbf{0.8532}$ & $\mathbf{0.7683}$ & 0.8210 \\

\hline
\end{tabular}
    \label{table:AD_result}
\end{table*}

\subsection{Anomaly Detection with Hypergraph U-Nets}\label{subsec:anomaly}

{\bf Dataset.} Introduced in~\cite{dou2020enhancing}, the Yelp and Amazon datasets have been widely used to evaluate node anomaly detection. The Yelp dataset treats hotel and restaurant reviews as nodes, and the task is to identify spam and legitimate reviews. It is a labeled dataset and the labels are retrieved from Yelp's filtering system. Each node has $32$ handcrafted features. The Yelp dataset contains three relationships: 1) R-U-R connects reviews posted by the same user, 2) R-S-R connects reviews
under the same product with the same star rating, 3) R-T-R connects reviews under the same product posted in the same month. 

The Amazon dataset includes users who have posted product reviews. A user is labeled as benign if he/she obtains more than $80\%$ helpful votes, and is labeled as fraudulent if there are less than $20\%$ helpful votes. One of the node features for the Amazon dataset is the minimum number of unhelpful votes, which introduces data leakage to predict anomaly users. As a result, we deleted this feature during data processing and the resulting $24$-dimensional node features.  The Amazon dataset has three relations: 1) U-P-U connects users reviewing at least one same product; 2) U-S-V connects users having at least one same star rating within one week; 3) U-V-U connects users with the top $5\%$ mutual review text similarities. In our implementation, we combine these three relationships for each dataset and treat them as homogeneous networks.

\begin{table}[htp]
\begin{center}
\caption{Summary Statistics of Anomaly Node Detection Datasets}
\resizebox{0.48\textwidth}{!}{
    \begin{tabular}{cccccc}
\hline
    & $\mathcal{V}$ & $\%$ anomaly & $|\mathcal{E}|$  & Feat Sim & Label Sim  \\
    \hline
Yelp  & $45,954$      & $ 14.5$ & 3,846,979  & $0.77$ & $0.07$  \\
Amazon & $11,944 $     & $9.5$ &4,398,392  & $0.65$ & $0.05$ \\
\hline
\end{tabular}}\label{table: data_stat2}
\end{center}
\end{table}

\noindent {\bf Baselines.} We choose $4$ types of baselines to compare with our model: classic GNNs, GNNs specifically developed for anomaly detection, graph u-nets, and HyperGNNs.

\noindent {\bf GNNs}: GCN~\cite{GCN} and GAT~\cite{GAT} are popular GNNs for homophilic graphs. GPR-GNN~\cite{chien2020adaptive} and FAGCN~\cite{bo2021beyond} are GNNs optimized for heterophilic graphs. GPR-GNN combines an adaptive generalized PageRank scheme with GNNs. FAGCN adaptively changes the proportion of low- and high-frequency signals with a self-gating mechanism.

\noindent {\bf GNNs-based Fraud Detection:} CARE-GNN~\cite{dou2020enhancing} and RioGNN~\cite{peng2021reinforced} are multi-relation GNNs designed for fraud detection that sample and aggregate features of similar neighbors. $\text{H}^2$-FDetector~\cite{shi2022h2} is a multi-relation GNN model that considers both homophilic and heterophilic connections. BWGNN~\cite{tang2022rethinking} proposes Beta Wavelet GNN for anomaly detection on multi-relation graphs.

{\bf Experiment setup.} Our main PHPool setting (used in the hypergraph-classification experiments of Section~\ref{subsec:hypergraph_classification}) selects layer-wise cluster assignments by modularity-guided thresholds and therefore does not require a manually specified pooling ratio. For the large-scale anomaly-detection benchmarks here, however, we adopt a simplified PHPool variant as a scalable approximation: instead of evaluating thresholds on the dendrogram and performing greedy selection, we set a fixed pooling ratio of $50\%$, reducing the number of nodes by half at each layer. This choice is based on the observation that, on hypergraphs of the size of Yelp and Amazon, modularity-refined threshold selection has minimal impact on cluster quality and downstream performance, while uniform halving substantially reduces preprocessing cost. Using the same split introduced by~\cite{dou2020enhancing}, we ran each experiment with different seeds for $10$ times and reported the average metrics. 

The baseline GNNs are implemented using the best hyperparameters from the original papers. The graph u-nets with TopKPooling~\cite{graphunet}, and MinCutPooling~\cite{bianchi2020spectral} are implemented with GCN as backbone models. To make a fair comparison, the depths of the graph u-nets and hypergraph u-nets are set to $3$, which means that there are two pooling layers and two corresponding unpooling layers. For hyperparameters, the hidden dimension is tuned among $\{32, 64, 128, 256\}$, and the dropout rate is either $0$ or $0.5$. The neural network is optimized with an Adam optimizer, with a learning rate $0.01$ for Yelp and $0.001$ for Amazon. The weight decay rate is $5\times 10^{-5}$ for both datasets.

{\bf Results and Discussion.} The results for node anomaly detection are presented in Table~\ref{table:AD_result}. HGUN outperforms other methods on the highly imbalanced Yelp and Amazon datasets in terms of precision, weighted F1-macro, and the geometric mean of recall and precision. This indicates that the node embeddings generated by HGUN, enhanced with pooling and unpooling operations, capture more discriminative node features. A possible reason is that varying the size of the hypergraph helps mitigate the problem of oversmoothing~\cite{predict_propagate} by expanding the receptive fields. However, HGUN slightly underperforms GNNs specifically designed for anomaly detection in terms of recall. This is expected, as anomaly-based GNNs often utilize additional information, such as node label heterogeneity, which is crucial for identifying anomalous nodes. It should also be noted that both HGXConv and PHPool leverage local node similarities. The lower recall suggests that HGUN may miss some true anomalies, likely due to the model's focus on capturing grouping structures within hypergraphs.

\subsection{Depth Sensitivity for Node Anomaly Detection}\label{subsec: depth_sensitivi}
We analyze the effect of pooling layer depth on model performance. As illustrated in Fig.~\ref{fig:depth_effect}, optimal performance occurs at moderate depths: depth 4 for Yelp (HGUN F1=0.7862) and depth 3 for Amazon (HGUN F1=0.7683). Beyond these points, all models experience performance degradation, with depth 6 causing significant drops across both datasets. This pattern suggests that while hierarchical pooling captures multi-scale graph structures effectively at moderate depths (3-4 layers), excessive depth leads to over-compression and loss of discriminative local structure. HGUN demonstrates superior depth robustness compared to MinCut and TopK, maintaining stable performance across depths 3-5. TopK exhibits the most severe degradation (Yelp: $0.730 \to 0.687$ from depth 2 to 6), indicating limited capacity for deep hierarchical learning. These results emphasize the critical role of depth selection in graph pooling architectures.
\begin{figure}[ht]
  \centering
  \includegraphics[width=\linewidth]{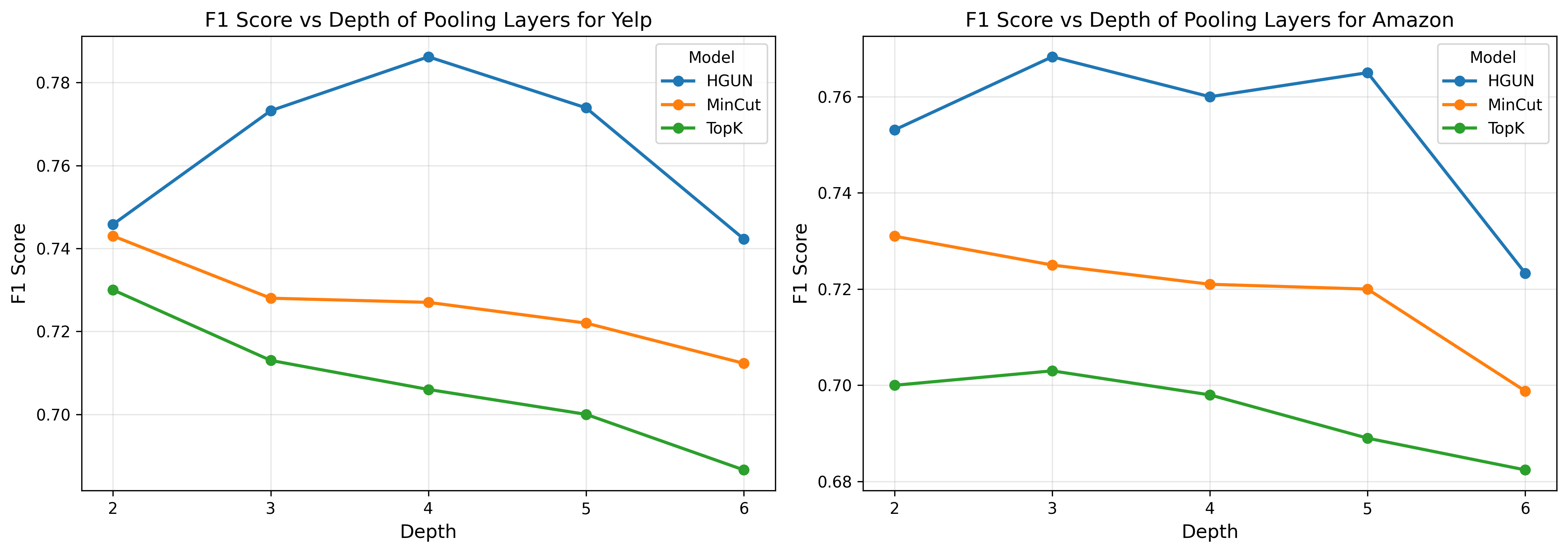}
  \vspace{-0.7cm}
\caption{F1 scores w.r.t depths on Yelp and Amazon datasets for node classification. Comparison between (hyper)graph autoencoders with various pooling.}
  \label{fig:depth_effect}
\end{figure}

{

\subsection{Runtime Comparison}\label{subsec: runtime_comparison}
  In addition to performance gains, hypergraph pooling offers computational efficiency benefits. Table~\ref{table: runtime} compares wall-clock runtime between HGUN and HGXConv on node anomaly detection. For fair comparison, HGXConv uses 6 convolutional layers to match HGUN's depth (3 PHPool + HGXConv blocks and 3 PHUnPool + HGXConv layers).
  HGUN achieves speedups of 13.13\% on Yelp and 18.82\% on Amazon. This efficiency stems from progressive graph coarsening, which reduces nodes and hyperedges in intermediate layers, lowering computational costs.
}

\begin{table}[htp]
\begin{center}
\caption{Wall-clock runtime per epoch (seconds) on node anomaly detection with matched depth. HGXConv uses no pooling; preprocessing time for PHPool is excluded.}
\resizebox{0.35\textwidth}{!}{
    \begin{tabular}{cccc}
\hline
   Dataset & HGUN & HGXConv & Speedup (\%)  \\
    \hline
Yelp  & 20.37      & 23.45 & 13.13\%   \\
Amazon & 6.73     &  8.29& 18.82\% \\
\hline
\end{tabular}}\label{table: runtime}
\end{center}
\end{table}

\section{Conclusion and Future Work}
In this work, we propose a new hypergraph neural network architecture: HGUN for node- and hypergraph-level tasks. The basic building blocks in the HGUN consist of hierarchical hypergraph pooling (PHPool) and hypergraph cross-convolution (HGXConv). The PHPool is formulated by greedily selecting clusters that produce the maximum modularity in a hierarchical dendrogram, constructing modularity-guided pooling layers that are robust to the number of layers. The HGXConv leverages cross-node multiplicative interactions and generalizes it to a node-hyperedge-node propagation scheme. Theoretically, we prove that the proposed PHPool at least retains the same expressiveness of the preceding convolution layers. Through experiments, we observe that HGUN achieves consistently competitive performance in hypergraph reconstruction, hypergraph classification, and node anomaly detection. In addition, we find that HGXConv achieves better performance than the other two-stage message-passing HyperGNNs, and the PHPool further promotes the expressiveness of the overall neural network. HGUN shows its potential to improve other hypergraph neural networks by investigating the autoencoder architecture.

\textbf{Limitations.} While HGUN demonstrates strong performance across diverse tasks, we note several aspects for future improvement. Our current approach employs hard clustering assignments determined in a preprocessing step, which ensures global optimization and computational efficiency but does not adapt to task-specific node features during training. Additionally, the $O(N^2)$ dendrogram construction can be expensive for very large hypergraphs, though approximation algorithms~\cite{moseley2023approximation, charikar2017approximate} and distributed methods~\cite{bateni2017affinity} can mitigate this cost. Finally, hypergraphs with overlapping community structures or unusual topology may benefit from soft clustering or alternative quality metrics beyond modularity.

  Interesting future work includes: 1) Learning soft, hierarchy-aware cluster assignments that explicitly incorporate node features when forming clusters. 2) Reducing the $O(N^2)$ computational cost of dendrogram construction by replacing precomputed clustering with trainable clustering assignment matrices optimized end-to-end under a global hierarchical objective.

\printbibliography

@article{HGNN,
  title={Hypergraph neural networks},
  author={Feng, Yifan and You, Haoxuan and Zhang, Zizhao and Ji, Rongrong and Gao, Yue},
  journal={Proceedings of the AAAI Conference on Artificial Intelligence},
  volume={33},
  number={01},
  pages={3558--3565},
  year={2019}
}

@article{wang2024t,
  title={T-HyperGNNs: Hypergraph neural networks via tensor representations},
  author={Wang, Fuli and Pena-Pena, Karelia and Qian, Wei and Arce, Gonzalo R},
  journal={IEEE Transactions on Neural Networks and Learning Systems},
  pages={1-15},
  year={2024},
doi={10.1109/TNNLS.2024.3371382},
  publisher={IEEE}
}

@article{tang2023learning,
  title={Learning hypergraphs from signals with dual smoothness prior},
  author={Tang, Bohan and Chen, Siheng and Dong, Xiaowen},
  journal={ICASSP 2023-2023 IEEE International Conference on Acoustics, Speech and Signal Processing (ICASSP)},
  pages={1--5},
  year={2023},
  organization={IEEE}
}

@article{SEP,
  title={Structural entropy guided graph hierarchical pooling},
  author={Wu, Junran and Chen, Xueyuan and Xu, Ke and Li, Shangzhe},
  journal={International conference on machine learning},
  pages={24017--24030},
  year={2022},
  organization={PMLR}
}

@article{errica2019fair,
  title={A fair comparison of graph neural networks for graph classification},
  author={Errica, Federico and Podda, Marco and Bacciu, Davide and Micheli, Alessio},
  journal={arXiv preprint arXiv:1912.09893},
  year={2019}
}

@article{dou2020enhancing,
  title={Enhancing graph neural network-based fraud detectors against camouflaged fraudsters},
  author={Dou, Yingtong and Liu, Zhiwei and Sun, Li and Deng, Yutong and Peng, Hao and Yu, Philip S},
  journal={Proceedings of the 29th ACM international conference on information \& knowledge management},
  pages={315--324},
  year={2020}
}

@article{ying2018hierarchical,
  title={Hierarchical graph representation learning with differentiable pooling},
  author={Ying, Zhitao and You, Jiaxuan and Morris, Christopher and Ren, Xiang and Hamilton, Will and Leskovec, Jure},
  journal={Advances in neural information processing systems},
  volume={31},
  year={2018}
}

@article{lee2019self,
  title={Self-attention graph pooling},
  author={Lee, Junhyun and Lee, Inyeop and Kang, Jaewoo},
  journal={International conference on machine learning},
  pages={3734--3743},
  year={2019},
  organization={PMLR}
}

@article{ranjan2020asap,
  title={Asap: Adaptive structure aware pooling for learning hierarchical graph representations},
  author={Ranjan, Ekagra and Sanyal, Soumya and Talukdar, Partha},
  journal={Proceedings of the AAAI conference on artificial intelligence},
  volume={34},
  number={04},
  pages={5470--5477},
  year={2020}
}

@article{bianchi2020spectral,
  title={Spectral clustering with graph neural networks for graph pooling},
  author={Bianchi, Filippo Maria and Grattarola, Daniele and Alippi, Cesare},
  journal={International conference on machine learning},
  pages={874--883},
  year={2020},
  organization={PMLR}
}

@article{chien2020adaptive,
  title={Adaptive universal generalized pagerank graph neural network},
  author={Chien, Eli and Peng, Jianhao and Li, Pan and Milenkovic, Olgica},
  journal={arXiv preprint:2006.07988},
  year={2020}
}

@article{bo2021beyond,
  title={Beyond low-frequency information in graph convolutional networks},
  author={Bo, Deyu and Wang, Xiao and Shi, Chuan and Shen, Huawei},
  journal={Proceedings of the AAAI conference on artificial intelligence},
  volume={35},
  number={5},
  pages={3950--3957},
  year={2021}
}

@article{peng2021reinforced,
  title={Reinforced neighborhood selection guided multi-relational graph neural networks},
  author={Peng, Hao and Zhang, Ruitong and Dou, Yingtong and Yang, Renyu and Zhang, Jingyi and Yu, Philip S},
  journal={ACM Transactions on Information Systems (TOIS)},
  volume={40},
  number={4},
  pages={1--46},
  year={2021},
  publisher={ACM New York, NY}
}

@article{shi2022h2,
  title={H2-fdetector: A gnn-based fraud detector with homophilic and heterophilic connections},
  author={Shi, Fengzhao and Cao, Yanan and Shang, Yanmin and Zhou, Yuchen and Zhou, Chuan and Wu, Jia},
  journal={Proceedings of the ACM Web Conference 2022},
  pages={1486--1494},
  year={2022}
}

@article{predict_propagate,
  title={Predict then propagate: Graph neural networks meet personalized pagerank},
  author={Gasteiger, Johannes and Bojchevski, Aleksandar and G{\"u}nnemann, Stephan},
  journal={arXiv preprint arXiv:1810.05997},
  year={2018}
}

@article{tang2022rethinking,
  title={Rethinking graph neural networks for anomaly detection},
  author={Tang, Jianheng and Li, Jiajin and Gao, Ziqi and Li, Jia},
  journal={International Conference on Machine Learning},
  pages={21076--21089},
  year={2022},
  organization={PMLR}
}

@article{mullner2011modern,
  title={Modern hierarchical, agglomerative clustering algorithms},
  author={M{\"u}llner, Daniel},
  journal={arXiv preprint arXiv:1109.2378},
  year={2011}
}

@article{bar2001fast,
  title={Fast optimal leaf ordering for hierarchical clustering},
  author={Bar-Joseph, Ziv and Gifford, David K and Jaakkola, Tommi S},
  journal={Bioinformatics},
  volume={17},
  number={suppl\_1},
  pages={S22--S29},
  year={2001},
  publisher={Oxford University Press}
}

@article{debnath1991structure,
  title={Structure-activity relationship of mutagenic aromatic and heteroaromatic nitro compounds. correlation with molecular orbital energies and hydrophobicity},
  author={Debnath, Asim Kumar and Lopez de Compadre, Rosa L and Debnath, Gargi and Shusterman, Alan J and Hansch, Corwin},
  journal={Journal of medicinal chemistry},
  volume={34},
  number={2},
  pages={786--797},
  year={1991},
  publisher={ACS Publications}
}

@article{konstantinova1998graph,
  title={Graph and hypergraph models of molecular structure: A comparative analysis of indices},
  author={Konstantinova, EV and Skoroboratov, VA},
  journal={Journal of structural chemistry},
  volume={39},
  number={6},
  pages={958--966},
  year={1998},
  publisher={Springer}
}

@article{kumar2020new,
  title={A new measure of modularity in hypergraphs: Theoretical insights and implications for effective clustering},
  author={Kumar, Tarun and Vaidyanathan, Sankaran and Ananthapadmanabhan, Harini and Parthasarathy, Srinivasan and Ravindran, Balaraman},
  journal={Proceedings of the Eighth International Conference on Complex Networks and Their Applications },
  pages={286--297},
  year={2020},
  organization={Springer}
}

@article{bianchi2024expressive,
  title={The expressive power of pooling in graph neural networks},
  author={Bianchi, Filippo Maria and Lachi, Veronica},
  journal={Advances in Neural Information Processing Systems},
  volume={36},
  year={2024}
}

@article{hyperg_isomorphism,
  title={Hypergraph isomorphism computation},
  author={Feng, Yifan and Han, Jiashu and Ying, Shihui and Gao, Yue},
  journal={IEEE Transactions on Pattern Analysis and Machine Intelligence},
  year={2024},
  publisher={IEEE}
}

@article{graphunet,
  title={Graph u-nets},
  author={Gao, Hongyang and Ji, Shuiwang},
  journal={international conference on machine learning},
  pages={2083--2092},
  year={2019},
  organization={PMLR}
}

@article{hypergraph_structure_learning,
  title={Hypergraph Structure Learning for Hypergraph Neural Networks.},
  author={Cai, Derun and Song, Moxian and Sun, Chenxi and Zhang, Baofeng and Hong, Shenda and Li, Hongyan},
  journal={IJCAI},
  pages={1923--1929},
  year={2022}
}

@article{hnhn,
  title={Hnhn: Hypergraph networks with hyperedge neurons},
  author={Dong, Yihe and Sawin, Will and Bengio, Yoshua},
  journal={arXiv preprint arXiv:2006.12278},
  year={2020}
}

@article{gao2022hgnn+,
  title={HGNN+: General hypergraph neural networks},
  author={Gao, Yue and Feng, Yifan and Ji, Shuyi and Ji, Rongrong},
  journal={IEEE Transactions on Pattern Analysis and Machine Intelligence},
  volume={45},
  number={3},
  pages={3181--3199},
  year={2022},
  publisher={IEEE}
}

@article{HCHA,
  title={Hypergraph convolution and hypergraph attention},
  author={Bai, Song and Zhang, Feihu and Torr, Philip HS},
  journal={Pattern Recognition},
  volume={110},
  pages={107637},
  year={2021},
  publisher={Elsevier}
}

@article{zhou2004regularization,
  title={A regularization framework for learning from graph data},
  author={Zhou, Dengyong and Sch{\"o}lkopf, Bernhard},
  journal={ICML 2004 Workshop on Statistical Relational Learning and Its Connections to Other Fields (SRL 2004)},
  pages={132--137},
  year={2004}
}

@article{GCN,
  title={Modeling relational data with graph convolutional networks},
  author={Schlichtkrull, Michael and Kipf, Thomas N and Bloem, Peter and Van Den Berg, Rianne and Titov, Ivan and Welling, Max},
  journal={European semantic web conference},
  pages={593--607},
  year={2018},
  organization={Springer}
}

@article{GraphSage,
  title={Inductive representation learning on large graphs},
  author={Hamilton, Will and Ying, Zhitao and Leskovec, Jure},
  journal={Advances in neural information processing systems},
  volume={30},
  year={2017}
}

@article{CE_learning,
  title={Learning with hypergraphs: Clustering, classification, and embedding},
  author={Zhou, Dengyong and Huang, Jiayuan and Sch{\"o}lkopf, Bernhard},
  journal={Advances in neural information processing systems},
  pages={1601--1608},
  year={2007}
}

@article{wang2023unified,
  title={A unified view between tensor hypergraph neural networks and signal denoising},
  author={Wang, Fuli and Pena-Pena, Karelia and Qian, Wei and Arce, Gonzalo R},
  journal={2023 31st European Signal Processing Conference (EUSIPCO)},
  pages={1968--1972},
  year={2023},
  organization={IEEE}
}

@article{cliquealgo,
  title={Algorithm 457: finding all cliques of an undirected graph},
  author={Bron, Coen and Kerbosch, Joep},
  journal={Communications of the ACM},
  volume={16},
  number={9},
  pages={575--577},
  year={1973},
  publisher={ACM New York, NY, USA}
}

@article{pena2023learning,
  title={Learning Hypergraphs Tensor Representations from Data via t-HGSP},
  author={Pena-Pena, Karelia and Taipe, Lucas and Wang, Fuli and Lau, Daniel L and Arce, Gonzalo R},
  journal={IEEE Transactions on Signal and Information Processing over Networks},
  year={2023},
  publisher={IEEE}
}

@article{morris2020tudataset,
  title={Tudataset: A collection of benchmark datasets for learning with graphs},
  author={Morris, Christopher and Kriege, Nils M and Bause, Franka and Kersting, Kristian and Mutzel, Petra and Neumann, Marion},
  journal={arXiv preprint arXiv:2007.08663},
  year={2020}
}

@article{ronneberger2015u,
  title={U-net: Convolutional networks for biomedical image segmentation},
  author={Ronneberger, Olaf and Fischer, Philipp and Brox, Thomas},
  journal={Medical image computing and computer-assisted intervention},
  pages={234--241},
  year={2015},
  organization={Springer}
}

@article{hypergraph_review,
  title={Signal processing on higher-order networks: Livin’on the edge... and beyond},
  author={Schaub, Michael T and Zhu, Yu and Seby, Jean-Baptiste and Roddenberry, T Mitchell and Segarra, Santiago},
  journal={Signal Processing},
  volume={187},
  pages={108149},
  year={2021},
  publisher={Elsevier}
}

@article{32,
  title={Orthogonal decomposition of symmetric tensors},
  author={Robeva, Elina},
  journal={SIAM Journal on Matrix Analysis and Applications},
  volume={37},
  number={1},
  pages={86--102},
  year={2016},
  publisher={SIAM}
}

@article{huang2021unignn,
  title={Unignn: a unified framework for graph and hypergraph neural networks},
  author={Huang, Jing and Yang, Jie},
  journal={arXiv preprint arXiv:2105.00956},
  year={2021}
}

@article{GAT,
  title={Graph attention networks},
  author={Veli{\v{c}}kovi{\'c}, Petar and Cucurull, Guillem and Casanova, Arantxa and Romero, Adriana and Lio, Pietro and Bengio, Yoshua},
  journal={arXiv preprint arXiv:1710.10903},
  year={2017}
}

@article{Allset,
  title={You are allset: A multiset function framework for hypergraph neural networks},
  author={Chien, Eli and Pan, Chao and Peng, Jianhao and Milenkovic, Olgica},
  journal={arXiv preprint arXiv:2106.13264},
  year={2021}
}

@article{bai2021hypergraph,
  title={Hypergraph convolution and hypergraph attention},
  author={Bai, Song and Zhang, Feihu and Torr, Philip HS},
  journal={Pattern Recognition},
  volume={110},
  pages={107637},
  year={2021},
  publisher={Elsevier}
}

@article{la2022music,
  title={Music recommendation via hypergraph embedding},
  author={La Gatta, Valerio and Moscato, Vincenzo and Pennone, Mirko and Postiglione, Marco and Sperl{\'\i}, Giancarlo},
  journal={IEEE Transactions on Neural Networks and Learning systems},
  year={2022},
  publisher={IEEE}
}

@article{wu2020comprehensive,
  title={A comprehensive survey on graph neural networks},
  author={Wu, Zonghan and Pan, Shirui and Chen, Fengwen and Long, Guodong and Zhang, Chengqi and Philip, S Yu},
  journal={IEEE Transactions on Neural Networks and Learning Systems},
  volume={32},
  number={1},
  pages={4--24},
  year={2020},
  publisher={IEEE}
}

@article{wang2020haar,
  title={Haar graph pooling},
  author={Wang, Yu Guang and Li, Ming and Ma, Zheng and Montufar, Guido and Zhuang, Xiaosheng and Fan, Yanan},
  journal={International conference on machine learning},
  pages={9952--9962},
  year={2020},
  organization={PMLR}
}

@article{austin_modularity,
  title={Generative hypergraph clustering: From blockmodels to modularity},
  author={Chodrow, Philip S and Veldt, Nate and Benson, Austin R},
  journal={Science Advances},
  volume={7},
  number={28},
  pages={eabh1303},
  year={2021},
  publisher={American Association for the Advancement of Science}
}

@article{feng2023modularity,
  title={Modularity-based Hypergraph Clustering: Random Hypergraph Model, Hyperedge-cluster Relation, and Computation},
  author={Feng, Zijin and Qiao, Miao and Cheng, Hong},
  journal={Proceedings of the ACM on Management of Data},
  volume={1},
  number={3},
  pages={1--25},
  year={2023},
  publisher={ACM New York, NY, USA}
}

@article{li2017inhomogeneous,
  title={Inhomogeneous hypergraph clustering with applications},
  author={Li, Pan and Milenkovic, Olgica},
  journal={Advances in neural information processing systems},
  volume={30},
  year={2017}
}

@article{he2016deep,
  title={Deep residual learning for image recognition},
  author={He, Kaiming and Zhang, Xiangyu and Ren, Shaoqing and Sun, Jian},
  journal={Proceedings of the IEEE conference on computer vision and pattern recognition},
  pages={770--778},
  year={2016}
}

@article{wang2020second,
  title={Second-order pooling for graph neural networks},
  author={Wang, Zhengyang and Ji, Shuiwang},
  journal={IEEE Transactions on Pattern Analysis and Machine Intelligence},
  volume={45},
  number={6},
  pages={6870--6880},
  year={2020},
  publisher={IEEE}
}

@article{dijkstra2022note,
  title={A note on two problems in connexion with graphs},
  author={Dijkstra, Edsger W},
  journal={Edsger Wybe Dijkstra: his life, work, and legacy},
  pages={287--290},
  year={2022}
}

@article{hao2021hypergraph,
  title={Hypergraph neural network for skeleton-based action recognition},
  author={Hao, Xiaoke and Li, Jie and Guo, Yingchun and Jiang, Tao and Yu, Ming},
  journal={IEEE Transactions on Image Processing},
  volume={30},
  pages={2263--2275},
  year={2021},
  publisher={IEEE}
}

@article{liu2020semi,
  title={Semi-Dynamic Hypergraph Neural Network for 3D Pose Estimation.},
  author={Liu, Shengyuan and Lv, Pei and Zhang, Yuzhen and Fu, Jie and Cheng, Junjin and Li, Wanqing and Zhou, Bing and Xu, Mingliang},
  journal={IJCAI},
  pages={782--788},
  year={2020}
}

@article{xia2021self,
  title={Self-supervised hypergraph convolutional networks for session-based recommendation},
  author={Xia, Xin and Yin, Hongzhi and Yu, Junliang and Wang, Qinyong and Cui, Lizhen and Zhang, Xiangliang},
  journal={Proceedings of the AAAI conference on artificial intelligence},
  volume={35},
  number={5},
  pages={4503--4511},
  year={2021}
}

@article{vinas2023hypergraph,
  title={Hypergraph factorization for multi-tissue gene expression imputation},
  author={Vi{\~n}as, Ramon and Joshi, Chaitanya K and Georgiev, Dobrik and Lin, Phillip and Dumitrascu, Bianca and Gamazon, Eric R and Li{\`o}, Pietro},
  journal={Nature machine intelligence},
  volume={5},
  number={7},
  pages={739--753},
  year={2023},
  publisher={Nature Publishing Group UK London}
}

@article{gilmer2017neural,
  title={Neural message passing for quantum chemistry},
  author={Gilmer, Justin and Schoenholz, Samuel S and Riley, Patrick F and Vinyals, Oriol and Dahl, George E},
  journal={International conference on machine learning},
  pages={1263--1272},
  year={2017},
  organization={PMLR}
}

@article{fortunato2016community,
  title={Community detection in networks: A user guide},
  author={Fortunato, Santo and Hric, Darko},
  journal={Physics reports},
  volume={659},
  pages={1--44},
  year={2016},
  publisher={Elsevier}
}

@article{dhulipala2021hierarchical,
  title={Hierarchical agglomerative graph clustering in nearly-linear time},
  author={Dhulipala, Laxman and Eisenstat, David and {\L}{\k{a}}cki, Jakub and Mirrokni, Vahab and Shi, Jessica},
  journal={International conference on machine learning},
  pages={2676--2686},
  year={2021},
  organization={PMLR}
}

@article{dhulipala2023terahac,
  title={Terahac: Hierarchical agglomerative clustering of trillion-edge graphs},
  author={Dhulipala, Laxman and {\L}{\k{a}}cki, Jakub and Lee, Jason and Mirrokni, Vahab},
  journal={Proceedings of the ACM on Management of Data},
  volume={1},
  number={3},
  pages={1--27},
  year={2023},
  publisher={ACM New York, NY, USA}
}

@article{duval2022higher,
  title={Higher-order clustering and pooling for graph neural networks},
  author={Duval, Alexandre and Malliaros, Fragkiskos},
  journal={Proceedings of the 31st ACM international conference on information \& knowledge management},
  pages={426--435},
  year={2022}
}

@article{eliasof2023haar,
  title={Haar wavelet feature compression for quantized graph convolutional networks},
  author={Eliasof, Moshe and Bodner, Benjamin J and Treister, Eran},
  journal={IEEE Transactions on Neural Networks and Learning Systems},
  volume={35},
  number={4},
  pages={4542--4553},
  year={2023},
  publisher={IEEE}
}

@article{zhong2023hierarchical,
  title={Hierarchical message-passing graph neural networks},
  author={Zhong, Zhiqiang and Li, Cheng-Te and Pang, Jun},
  journal={Data Mining and Knowledge Discovery},
  volume={37},
  number={1},
  pages={381--408},
  year={2023},
  publisher={Springer}
}

@article{vonessen2024next,
  title={Next level message-passing with hierarchical support graphs},
  author={Vonessen, Carlos and Gr{\"o}tschla, Florian and Wattenhofer, Roger},
  journal={arXiv preprint arXiv:2406.15852},
  year={2024}
}

@article{finder2025improving,
  title={Improving the Effective Receptive Field of Message-Passing Neural Networks},
  author={Finder, Shahaf E and Weber, Ron Shapira and Eliasof, Moshe and Freifeld, Oren and Treister, Eran},
  journal={arXiv preprint arXiv:2505.23185},
  year={2025}
}

@article{charikar2017approximate,
  title={Approximate hierarchical clustering via sparsest cut and spreading metrics},
  author={Charikar, Moses and Chatziafratis, Vaggos},
  journal={Proceedings of the Twenty-Eighth Annual ACM-SIAM Symposium on Discrete Algorithms},
  pages={841--854},
  year={2017},
  organization={SIAM}
}

@article{manghiuc2021hierarchical,
  title={Hierarchical Clustering: $ O (1) $-Approximation for Well-Clustered Graphs},
  author={Manghiuc, Bogdan-Adrian and Sun, He},
  journal={advances in neural information processing systems},
  volume={34},
  pages={9278--9289},
  year={2021}
}

@article{moseley2023approximation,
  title={Approximation bounds for hierarchical clustering: Average linkage, bisecting k-means, and local search},
  author={Moseley, Benjamin and Wang, Joshua R},
  journal={Journal of Machine Learning Research},
  volume={24},
  number={1},
  pages={1--36},
  year={2023}
}

@article{bateni2017affinity,
  title={Affinity clustering: Hierarchical clustering at scale},
  author={Bateni, MohammadHossein and Behnezhad, Soheil and Derakhshan, Mahsa and Hajiaghayi, MohammadTaghi and Kiveris, Raimondas and Lattanzi, Silvio and Mirrokni, Vahab},
  journal={Advances in Neural Information Processing Systems},
  volume={30},
  year={2017}
}
\vspace*{-3cm}
\begin{IEEEbiography}
[{\includegraphics[width=0.9in,height=1.25in,clip,keepaspectratio]{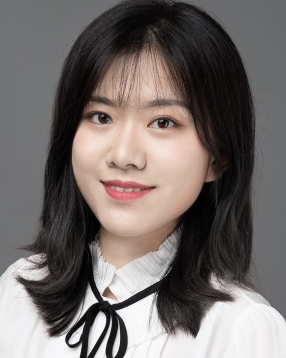}}]{Fuli Wang} received her M.Sc. degree in Statistics from the University of Minnesota in 2020, and her Ph.D degree in Financial Services Analytics from Institute of Financial Services at the University of Delaware in 2024. Her research interests include hypergraph neural networks, hypergraph signal processing, and their applications in finance and health. 
\end{IEEEbiography}

\newpage
\begin{IEEEbiography}[{\includegraphics[width=0.9in,height=1.25in,clip,keepaspectratio]{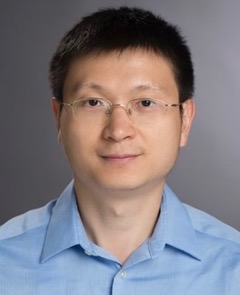}}]{Wei Qian} received his Ph.D. degree in Statistics from the University of
Minnesota, Minneapolis, MN, in 2014. He joined the School of Mathematics and
Statistics at Rochester Institute of Technology as an assistant professor in the
same year, and then moved to the Department of Applied Economics and
Statistics at the University of Delaware in 2017, where he has been an associate
professor since 2021. His research interests include high-dimensional statistics, model selection, dimension reduction, statistical computing, deep learning,
reinforcement learning, and data science applications. He also serves as a JPMC
Fellow and affiliated member of the Institute of Financial Services at the
University of Delaware.
\end{IEEEbiography}

\vspace*{-10cm}
\begin{IEEEbiography}[{\includegraphics[width=0.9in,height=1.25in,clip,keepaspectratio]{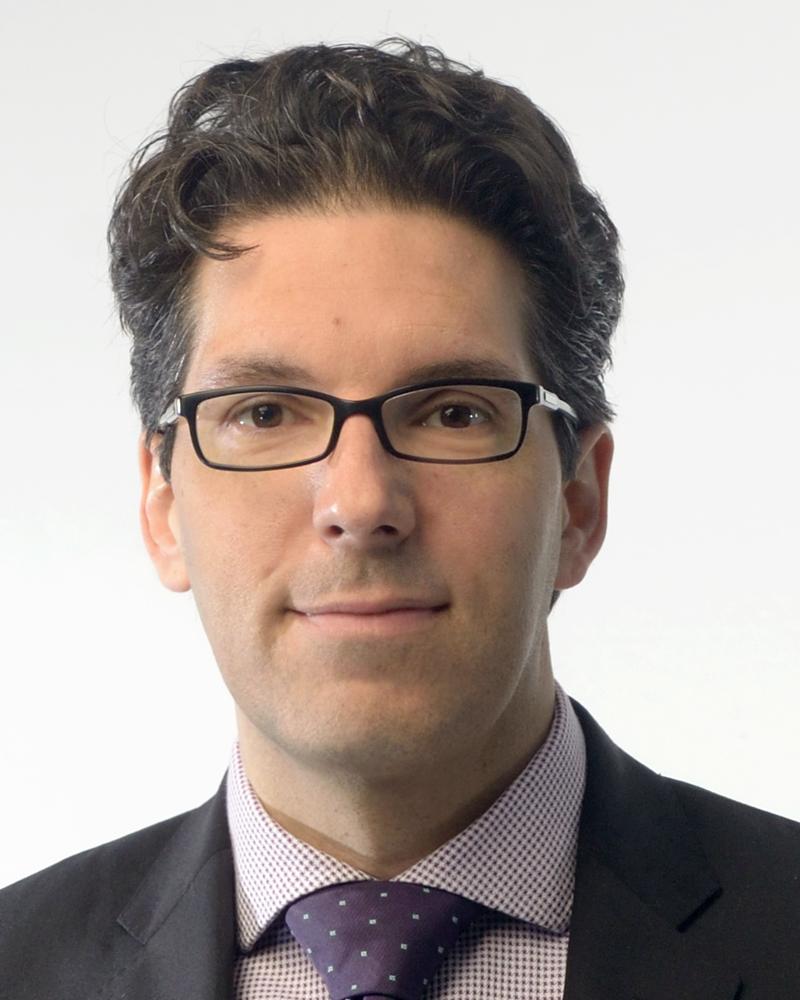}}]{Daniel L. Lau (Fellow, IEEE)} received the B.Sc. degree (with highest distinction) in electrical engineering from Purdue University, West Lafayette, IN, USA, and the Ph.D. degree from the University of Delaware, Newark, DE, USA, in 1995 and 1999, respectively. He is currently a Professor and the Director of Graduate Studies at the University of Kentucky's Department of Electrical and Computer Engineering, Lexington, KY, USA. His research interests include image and signal processing, 3D imaging, and machine vision. His work has also been featured in trade magazines such as Imaging Insight, Prosilica Camera News, and Inspect Magazine.
\end{IEEEbiography}

\vspace*{-10cm}
\begin{IEEEbiography}[{\includegraphics[width=0.9in,height=1.25in,clip,keepaspectratio]{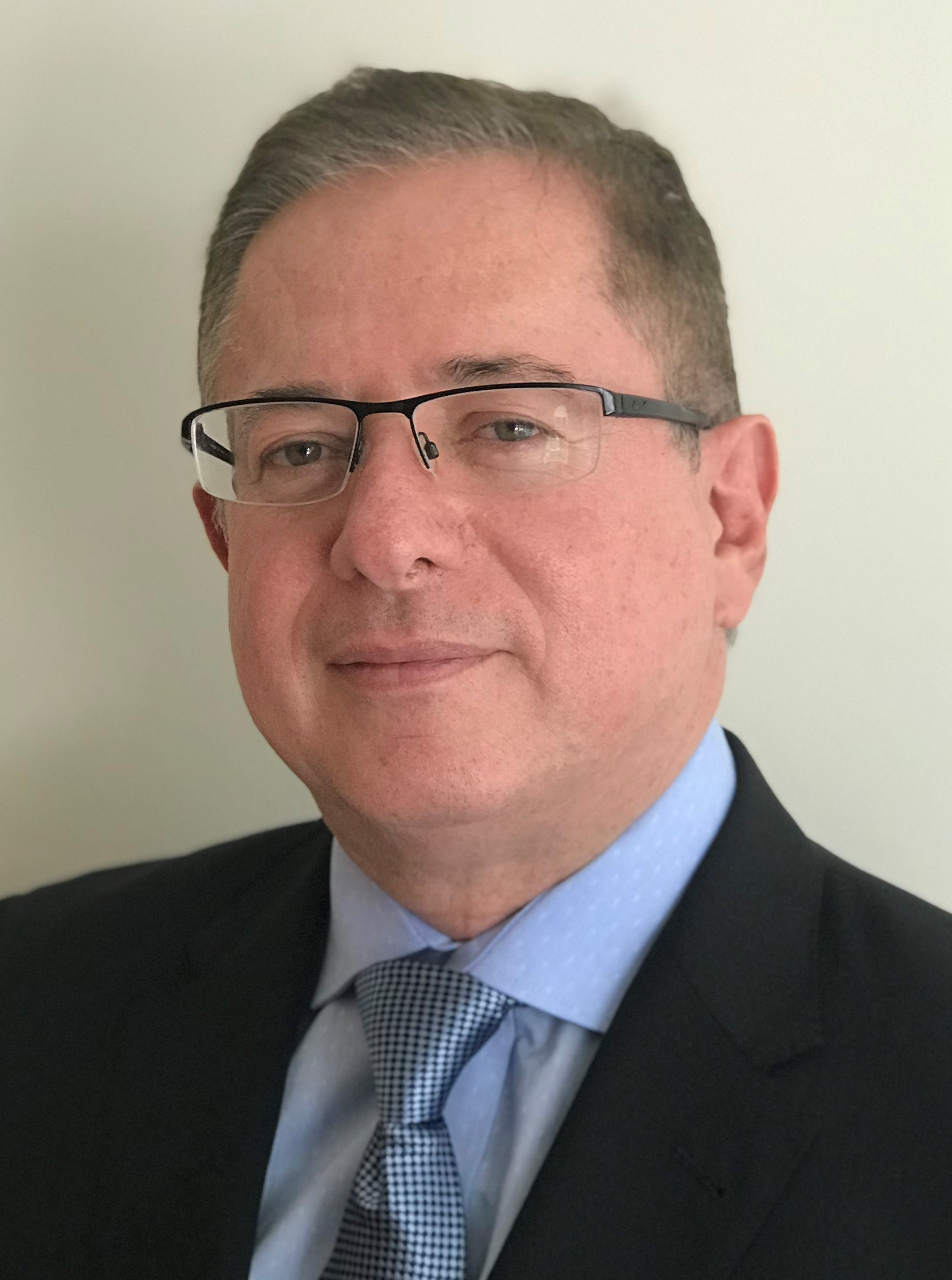}}]{Gonzalo R. Arce (Life Fellow, IEEE)}
	is currently the Charles Black Evans Distinguished Professor of Electrical and Computer engineering, and a J.P. Morgan-Chase Senior Faculty Fellow in the Institute of Financial Services Analytics at the University of Delaware, Newark, DE, USA. He holds 25 U.S. patents and is the co-author of four books. His research interests include computational imaging, data science, and machine learning. He held the 2010 and 2017 Fulbright-Nokia Distinguished Chair of Information and Communications Technologies with Aalto University, Espoo, Finland. He received the NSF Research Initiation Award. He was elected Fellow of the IEEE, OPTICA, SPIE, AAIA, and the National Academy of Inventors (NAI).
\end{IEEEbiography}

\end{document}